%% file: paper.tex
\documentclass[]{article}
\usepackage{proceed2e}

\usepackage{amsfonts}
\usepackage{amsmath}
\usepackage{amsthm}

\usepackage{float}

\usepackage{calc}

\usepackage{graphicx, wrapfig, MnSymbol, hyperref}
\usepackage{balance, multirow, algorithm, algorithmic, framed, subfig, caption}  
\usepackage{aliascnt}
\usepackage{relsize}

\input{macros}

\newlength\listingnumberwidth
\usepackage{listings}
\usepackage[round]{natbib}

\usepackage[usenames,dvipsnames]{color}
\usepackage{xspace}

\usepackage{tikz}
\usetikzlibrary{calc}

\usepackage[shortlabels]{enumitem}
\setlist{nolistsep} 

\let\citet\citep

\title{Understanding the Complexity of Lifted Inference \\ and  Asymmetric Weighted Model Counting}

\author{ 
{\bf Eric Gribkoff }\\
University of Washington \\
eagribko@cs.uw.edu \\
\And
{\bf Guy Van den Broeck}  \\
KU Leuven, UCLA \\
guyvdb@cs.ucla.edu \\
\And
{\bf Dan Suciu}   \\
University of Washington \\
suciu@cs.uw.edu \\
}

\newcommand{\algoname}{\textbf{Lift}$^{\textbf{R}}$}

\DeclareMathOperator*{\WFOMC}{WFOMC}
\DeclareMathOperator{\WMC}{WMC}

\newcommand{\pred}[1]{\mathtt{#1}}

\newcommand{\prof}{\pred{Prof}}
\newcommand{\advises}{\pred{Advises}}
\newcommand{\student}{\pred{Student}}

\newcommand{\anne}{\mathsf{Anne}}
\newcommand{\bob}{\mathsf{Bob}}
\newcommand{\charlie}{\mathsf{Charlie}} 

\def\Pr{\mathop{\rm Pr}\nolimits}
\newcommand{\w}{w}
\newcommand{\wt}{w^\star}
\newcommand{\wf}{\bar{w}^\star}

\begin{document}

\setlength{\belowdisplayskip}{0.5em}
\setlength{\belowdisplayshortskip}{0.5em}
\setlength{\abovedisplayskip}{0.5em}
\setlength{\abovedisplayshortskip}{0.5em}

\maketitle

\begin{abstract}

In this paper we study lifted inference for the Weighted First-Order Model Counting problem (WFOMC), which counts the assignments that satisfy a given sentence in first-order logic (FOL); it has applications in Statistical Relational Learning (SRL) and Probabilistic Databases (PDB).  We present several results.  First, we describe a lifted inference algorithm that generalizes prior approaches in SRL and PDB.  Second, we provide a novel dichotomy result for a non-trivial fragment of FO CNF sentences, showing that for each sentence the WFOMC problem is either in PTIME or \#P-hard in the size of the input domain; we prove that, in the first case our algorithm solves the WFOMC problem in PTIME, and in the second case it fails.  Third, we present several properties of the algorithm.  Finally, we discuss limitations of lifted inference for symmetric probabilistic databases (where the weights of ground literals depend only on the relation name, and not on the constants of the domain), and prove the impossibility of a dichotomy result for the complexity of probabilistic inference for the entire language FOL.

\end{abstract}

\input{introduction}

\input{background} 

\input{algorithm}

\input{mainResults}

\input{examples}

\input{extensions}

\input{exampleProof}

\input{relatedWork}

\input{conclusion}

\subsubsection*{Acknowledgements}
This work was partially supported by ONR grant \#N00014-12-1-0423, NSF grants IIS-1115188 and IIS-1118122, and the Research Foundation-Flanders (FWO-Vlaanderen).

\bibliographystyle{plainnat}
\bibliography{references}

\include{appendix}

\end{document}

%% file: macros.tex
\newcommand{\true}{{\tt true}}
\newcommand{\false}{{\tt false}}

\newcommand{\proj}[1]{{\Pi}}
\newcommand{\sel}[1]{{\sigma}}

\newcommand{\cut}[1]{}

\newcommand{\commentresolved}[1]{}

\newcommand{\set}[1]{\{#1\}}                    
\newtheorem{theorem}{Theorem}[section]          	
\newaliascnt{lemma}{theorem}				
\newtheorem{lemma}[lemma]{Lemma}              	
\aliascntresetthe{lemma}  					
\newaliascnt{conjecture}{theorem}			
\aliascntresetthe{conjecture}  				
\newaliascnt{remark}{theorem}				
              
\aliascntresetthe{remark}  					
\newaliascnt{corollary}{theorem}			
\newtheorem{corollary}[corollary]{Corollary}      
\aliascntresetthe{corollary}  				
\newaliascnt{definition}{theorem}			
\newtheorem{definition}[definition]{Definition}    
\aliascntresetthe{definition}  				
\newaliascnt{proposition}{theorem}			
\newtheorem{proposition}[proposition]{Proposition}  
\aliascntresetthe{proposition}  				
\newaliascnt{example}{theorem}			
\aliascntresetthe{example}  				
\newaliascnt{observation}{theorem}			
\aliascntresetthe{observation}  				
\newaliascnt{fact}{theorem}			
\newtheorem{fact}[fact]{Fact}  
\aliascntresetthe{fact}  				

\cut{

}



%% file: introduction.tex
\section{INTRODUCTION}
\label{sec:intro}
Weighted model counting (WMC) is a problem at the core of many reasoning tasks. It is based on the model counting or \#SAT task~\citep{gomes2009model}, where the goal is to count assignments that satisfy a given logical sentence. 
WMC generalizes model counting by assigning a weight to each assignment, and computing the \emph{sum of their weights}.
WMC has many applications in AI and its importance is increasing. Most notably, it underlies state-of-the-art probabilistic inference algorithms for Bayesian networks~\citep{darwiche2002logical,sang2005solving,Chavira2008}, relational Bayesian networks~\citep{Chavira2006} and probabilistic programs~\citep{Fierens11}.

This paper is concerned with weighted \emph{first-order} model counting (WFOMC), where we sum the weights of assignments that satisfy a sentence in finite-domain first-order logic.
Again, this reasoning task underlies efficient algorithms for probabilistic reasoning, this time for popular representations in statistical relational learning~(SRL)~\citep{Getoor07:book}, such as Markov logic networks~\citep{DBLP:conf/ijcai/BroeckTMDR11,gogatePTP} and probabilistic logic programs~\citep{VanDenBroeckMD14}. 
Moreover, WFOMC uncovers a deep connection between AI and database research, where query evaluation in \emph{probabilistic databases}~(PDBs)~\citep{suciu2011probabilistic} essentially considers the same task. A PDB defines a probability, or weight, for every possible world, and each database query is a sentence encoding a set of worlds, whose combined  probability we want to compute.

Early on, the disconnect between compact relational representations of uncertainty, and the intractability of inference at the ground, propositional level was noted, and efforts were made to exploit the relational structure for inference, using so-called \emph{lifted inference} algorithms~\citep{poole2003first,Kersting:2012}.
SRL and PDB algorithms for WFOMC all fall into this category.
Despite these commonalities, there are also important differences.
SRL has so far considered \emph{symmetric} WFOMC problems, where relations of the same type are assumed to contribute equally to the probability of a world. This assumption holds for certain queries on SRL models, such as single marginals and partition functions, but fails for more complex conditional probability queries. These break lifted algorithms based on symmetric WFOMC~\citep{van2013complexity}.
PDBs, on the other hand, have considered the  \emph{asymmetric} WFOMC setting from the start. Probabilistic database tuples have distinct probabilities, many have probability zero, and no symmetries can be expected.
However, current asymmetric WFOMC algorithms \citep{dalvi2012dichotomy} suffer from a major limitation of their own, in that they can only count models of sentences in monotone disjunctive normal form~(MDNF) (i.e., DNF without negation). Such sentences represent unions of conjunctive database queries~(UCQ).
WFOMC encodings of SRL models almost always fall outside this class.

The present work seeks to upgrade a well-known PDB algorithm for asymmetric WFOMC~\citep{dalvi2012dichotomy} to the SRL setting, by enabling it to count models of arbitrary sentences  in conjunctive normal form~(CNF). This permits its use for lifted SRL inference with arbitrary soft or hard evidence, or equivalently, probabilistic database queries with negation.
Our first contribution is this algorithm, which we call \algoname, and is presented in Section~\ref{sec:algorithm}.

Although \algoname\ has clear practical merits, we are in fact motivated by fundamental theoretical questions. 
In the PDB setting, our algorithm is known to come with a sharp complexity guarantee, called the \emph{dichotomy theorem}~\citep{dalvi2012dichotomy}. By only looking at the structure of the first-order sentence (i.e., the database query), the algorithm reports failure when the problem is \#P-hard (in terms of data complexity), and otherwise guarantees to solve it in time polynomial in the domain (i.e., database) size. It can thus precisely classify MDNF sentences as being tractable or intractable for asymmetric WFOMC.
Whereas several complexity results for symmetric WFOMC exist~\citep{van2011completeness,jaeger2012liftability}, the complexity of asymmetric WFOMC for SRL queries with evidence is still poorly understood.
Our second and main contribution, presented in Section~\ref{sec:mainresults}, is a \emph{novel dichotomy result} over a small but non-trivial fragment of CNFs.
We completely classify this class of problems as either computable in polynomial time or \#P-hard.
This represents a first step towards proving the following conjecture: \algoname\ provides a dichotomy for asymmetric WFOMC on arbitrary CNF sentences, and therefore perfectly classifies all related SRL models as tractable or intractable for conditional queries.

As our third contribution, presented in Section~\ref{sec:examples}, we illustrate the algorithm with examples that show its application to common probabilistic models. We discuss the capabilities of \algoname\ that are not present in other lifted inference techniques. 

As our fourth and final contribution, in Section~\ref{sec:extensions}, 
we discuss extensions of our algorithm to symmetric WFOMC, but also show the impossibility of a dichotomy result for arbitrary first-order logic sentences.

%% file: background.tex

\section{BACKGROUND}

We begin by introducing the necessary background on relational logic and weighted model counting.

\subsection{RELATIONAL LOGIC}

Throughout this paper, we will work with the relational fragment of first-order logic~(FOL), which we now briefly review.
An {\em atom} $P(t_1, \dots , t_n)$ consists of predicate $P/n$ of arity $n$ followed by $n$ arguments, which are either \textit{constants} or \textit{logical variables} $\{x, y, \dots\}$.
A \emph{literal} is an atom or its negation.
A \emph{formula} combines atoms with logical connectives and quantifiers $\exists$ and $\forall$.
A \emph{substitution} $[a/x]$ replaces all occurrences of $x$ by $a$. Its application to formula $F$ is denoted~$F[a/x]$.
A formula is a sentence if each logical variable $x$ is enclosed by a $\forall x$ or $\exists x$. 
A formula is {\em ground} if it contains no logical variables.
A \emph{clause} is a universally quantified disjunction of literals.
A \emph{term} is an existentially quantified conjunction of literals.
A \emph{CNF} is a conjunction of clauses, and a \emph{DNF} is a disjunction of terms. A \emph{monotone} CNF or DNF contains no negation symbols.
As usual, we drop the universal quantifiers from the CNF syntax.

The semantics of sentences are defined in the usual way~\citep{hinrichs2006herbrand}. An interpretation, or world, $I$ that satisfies sentence $\Delta$ is denoted by~$I~\models~\Delta$, and represented as a set of literals.
Our algorithm checks properties of sentences that are undecidable in general FOL, but decidable, with the following complexity, in the CNF fragment we investigate.
\begin{theorem} \label{th:equivalence}
  \citep{sagiv1980equivalences} 
  Checking whether logical implication $Q
  \Rightarrow Q'$ or equivalence $Q \equiv Q'$ holds between two CNF sentences is $\Pi^p_2$-complete.
\end{theorem}

\subsection{WEIGHTED MODEL COUNTING}

Weighted model counting was introduced as a propositional reasoning problem.
\begin{definition}[WMC]
  Given 
  a propositional \emph{sentence} $\Delta$ over literals $\mathcal{L}$, and 
  a \emph{weight function} $\w: \mathcal{L} \rightarrow \mathbb{R}^{\geq0}$,
  the \emph{weighted model count} (WMC) is
  \begin{align*}
    &\WMC(\Delta,\w) = \sum_{I \models \Delta} \,\, \prod_{\ell \in I} \, \w(\ell).
  \end{align*}
\end{definition}
We will consider its generalization to \emph{weighted first-order model counting}~(WFOMC), where $\Delta$ is now a sentence in relational logic, and $\mathcal{L}$ consists of all ground first-order literals for a given domain of constants.

The WFOMC task captures query answering in probabilistic database. Take for example the database
\begin{align*}
  \prof(\anne) : 0.9 & & \prof(\charlie) : 0.1 \\
  \student(\bob) : 0.5 & & \student(\charlie) : 0.8 \\
  \advises(\anne,\bob) : 0.7 & & \advises(\bob,\charlie) : 0.1
\end{align*}
and the UCQ (monotone DNF) query
\begin{align*}
  Q \,\, = \,\, \exists x, \exists y,\, \prof(x) \land \advises(x,y) \land \student(y).
\end{align*}
If we set $\Delta = Q$ and $\w$ to map each literal to its probability in the database, then our query answer is 
\begin{align*}
  \Pr(Q) = \WFOMC(\Delta,\w) = 0.9 \cdot 0.7 \cdot 0.5 = 0.315.
\end{align*}
We refer to the general case above as \emph{asymmetric} WFOMC, because it allows $\w(\prof(\anne))$ to be different from $\w(\prof(\charlie))$.
We use \emph{symmetric} WFOMC to refer to the special case where $\w$ simplifies into two weight functions $\wt,\wf$ that map \emph{predicates} to weights, instead of literals, that is
\begin{align*}
 \w(\ell) = \begin{cases}
              \wt(P) & \text{when $\ell$ is of the form $P(c)$} \\
              \wf(P) & \text{when $\ell$ is of the form $\neg P(c)$}
            \end{cases}
\end{align*}
Symmetric WFOMC no longer directly captures PDBs. Yet it can still encode many SRL models, including parfactor graphs~\citep{poole2003first}, Markov logic networks (MLNs)~\citep{richardson2006markov} and probabilistic logic programs~\citep{DeRaedt2008-PILP}.
We refer to \citet{VanDenBroeckMD14} for the details, and show here the following example MLN.
\begin{align*}
  2 \quad \prof(x) \land \advises(x,y) \Rightarrow \student(y)
\end{align*}
It states that the probability of a world increases by a factor $e^2$ with every pair of people $x,y$ for which the formula holds. Its WFOMC encoding has $\Delta$ equal to
\begin{align*}
  &\forall x, \forall y,\, \pred{F}(x,y) \Leftrightarrow \\
  &  \quad \qquad \left[\prof(x) \land \advises(x,y) \Rightarrow \student(y) \right]
\end{align*}
and weight functions $\wt,\wf$ such that $\wt(\pred{F})=e^2$ and all other predicates map to $1$.

Answering an SRL query $Q$ given evidence $E$, that is, $\Pr(Q\,|\,E)$, using a symmetric WFOMC encoding, generally requires solving two WFOMC tasks:
\begin{align*}
\Pr(Q\,|\,E) = \frac{\WFOMC(Q \land E \land \Delta, \w)}{\WFOMC(E \land \Delta, \w)}
\end{align*}
Symmetric WFOMC problems are strictly more tractable than asymmetric ones. We postpone the discussion of this observation to Section~\ref{sec:examples}, but already note that all theories $\Delta$ with up to two logical variables per formula support \emph{domain-lifted} inference~\citep{van2011completeness}, which means that any WFOMC query runs in time polynomial in the domain size (i.e, number of constants). 
For conditional probability queries, even though fixed-parameter complexity bounds exist that use symmetric WFOMC~\citep{van2013complexity}, the actual underlying reasoning task is asymmetric WFOMC, whose complexity we investigate for the first time.

Finally, we make three simplifying observations. 
First, SRL query $Q$ and evidence $E$ typically assign values to random variables. This means that the query and evidence can be absorbed into the asymmetric weight function, by setting the weight of literals disagreeing with $Q$ or $E$ to zero.
We hence compute:
\begin{align*}
\Pr(Q\,|\,E) = \frac{\WFOMC(\Delta, \w_{Q \land E})}{\WFOMC(\Delta, \w_{E})}
\end{align*}
This means that our complexity analysis for a given encoding $\Delta$ applies to both numerator and denominator for arbitrary $Q$ and $E$, and that polytime WFOMC for $\Delta$ implies polytime $\Pr(Q\,|\,E)$ computation. The converse is not true, since it is possible that both WFOMC calls are \#P-hard, but their ratio is in PTIME.
Second, we will from now on assume that $\Delta$ is in CNF. 
The WFOMC encoding of many SRL formalisms is already in CNF, or can be reduced to it~\citep{VanDenBroeckMD14}.
For PDB queries that are in monotone DNF, we can simply compute $\Pr(Q) = 1 - \Pr(\neg Q)$, which reduces to WFOMC on a CNF. Moreover, by adjusting the probabilities in the PDB, this CNF can also be made monotone.
Third, we will assume that $\w(\ell) = 1 - \w(\neg \ell)$, which can always be achieved by normalizing the weights.

Under these assumptions, we can simply refer to $\WFOMC(Q,\w)$ as $\Pr(Q)$, to $Q$ as the CNF query, to $\w(\ell)$ as the probability $\Pr(\ell)$, and to the entire weight function $\w$ as the PDB.
This is in agreement with notation in the PDB literature.

%% file: algorithm.tex
\section{ALGORITHM \algoname}
\label{sec:algorithm}

We present here the lifted algorithm \algoname{} (pronounced {\em
  lift-ER}), which, given a CNF formula $Q$ computes $\Pr(Q)$ in
polynomial time in the size of the PDB, or fails.  In the next
section we provide some evidence for its completeness: under certain
assumptions, if \algoname{} fails on formula $Q$, then computing
$\Pr(Q)$ is \#P-hard in the PDB size.

\subsection{DEFINITIONS}

An \emph{implicate} of $Q$ is some clause $C$ s.t.\ the logical
implication $Q \Rightarrow C$ holds.  $C$ is a \emph{prime implicate}
if there is no other implicate $C'$ s.t. $C' \Rightarrow C$.

A \emph{connected component} of a clause $C$ is a minimal subset of
its atoms that have no logical variables in common with the rest of
the clause. If some prime implicate $C$ has more than one connected component, then we can write it as:
\[
C = D_1 \lor D_2 \lor \cdots \lor D_m
\]
where each $D_i$ is a clause with distinct variables. 
Applying distributivity, we write $Q$ in \emph{union-CNF} form:
\[
Q = Q_1 \lor Q_2 \lor \cdots \lor Q_m
\]
where each $Q_i$ is a CNF with distinct variables.

We check for disconnected prime implicates $D_1 \lor D_2$ where both
$D_1$ and $D_2$ subsume some clause of $Q$. Intuitively, this means
that when we apply inclusion/exclusion to the union-CNF, the resulting
queries are simpler.  The search for $D_1$, $D_2$ can proceed using
some standard inference algorithm, e.g.~resolution. By
\autoref{th:equivalence}, this problem is $\Pi^p_2$-complete in the
size of the query $Q$, but independent of the PDB size.

A set of \emph{separator variables} for a query $Q = \bigwedge_{i=1}^k
C_i$ is a set of variables $x_i, i=1,k$ such that, (a) for each clause
$C_i$, $x_i$ occurs in all atoms of $C_i$, and (b) any two atoms (not
necessarily in the same clause) referring to the same relation $R$
have their separator variable on the same position.

\subsection{PREPROCESSING}

\label{subsec:preproc}

We start by transforming $Q$ (and PDB) such that:
\begin{enumerate}
	\item No constants occur in $Q$.
	\item If all the variables in $Q$ are $x_1, x_2, \dots, x_k$,
          then every relational atom in $Q$ (positive or negated) is
          of the form $R(x_{i_1}, x_{i_2}, \dots)$ such that $i_1 <
          i_2 < \dots$
\end{enumerate}

Condition (1) can be enforced by \emph{shattering} $Q$ w.r.t.\ its variables.
Condition (2) can be enforced by modifying both the query Q and the
database, in a process called {\em ranking} and described in the
appendix.  Here, we illustrate ranking on an example.  Consider the
query:
\begin{align*}
  Q = (R(x,y) \lor S(x,y)) \land (\neg R(x,y) \lor \neg S(y,x))
\end{align*}
Define $R_1(x,y) \equiv R(x,y)\wedge (x<y)$; $R_2(x) \equiv R(x,x)$;
$R_3(y,x) \equiv R(x,y)\wedge (x>y)$. Define similarly $S_1, S_2,
S_3$.  Given a PDB with relations $R$, $S$, we define a new PDB$'$ over
the six relations by setting $\Pr(R_1(a,b)) = \Pr(R(a,b))$ when $a<b$,
$\Pr(R_1(a,b)) = 0$ when $a > b$, $\Pr(R_2(a))= \Pr(R(a,a))$, etc.
Then, the query $Q$ over PDB is equivalent to the following query over
PDB':
\begin{align*}
 & (R_1(x,y) \lor S_1(x,y)) \land  (\neg R_1(x,y) \lor \neg S_3(x,y))\\
 & (R_2(x) \lor S_2(x)) \land (\neg R_2(x) \lor \neg S_2(x)) \\
 & (R_3(x,y) \lor S_3(x,y)) \land (\neg R_3(x,y) \lor \neg S_1(x,y))
\end{align*}

\subsection{ALGORITHM DESCRIPTION}

Algorithm \algoname{}, given in Figure \ref{fig:alg}, proceeds
recursively on the structure of the CNF query $Q$.  When it reaches
ground atoms, it simply looks up their probabilities in the PDB.
Otherwise, it performs the following sequence of steps.

First, it tries to express $Q$ as a union-CNF.
If it succeeds, and if the union can be partitioned into two sets 
that do not share any relational symbols, $Q = Q_1 \lor Q_2$, then 
it applies a {\em
  Decomposable Disjunction}:
\begin{gather*}
  \Pr(Q) = 1 - (1 - \Pr(Q_1)(1 - \Pr(Q_2))
\end{gather*}
Otherwise, it applies the {\em Inclusion/Exclusion} formula:
\begin{align*}
  \Pr(Q) = - \sum_{s \subseteq [m]} (-1)^{|s|} \Pr(\bigwedge_{i\in s}  Q_i)
\end{align*}
However, before computing the recursive probabilities, our algorithm
first checks for equivalent expressions, i.e.\ it checks for terms
$s_1, s_2$ in the inclusion/exclusion formula such that $\bigwedge_{i
  \in s_1} Q_i \equiv \bigwedge_{i \in s_2} Q_i$: in that case, these
terms either cancel out, or add up (and need be computed only
once).  We show in \autoref{subsec:cancellations} the critical role that the
cancellation step plays for the completeness of the algorithm.  To
check cancellations, the algorithm needs to check for equivalent CNF
expressions. This can be done using some standard inference algorithm
(recall from \autoref{th:equivalence} that this problem is
$\Pi^p_2$-complete in the size of the CNF expression).

If neither of the above steps apply, then the algorithm checks if
$Q$ can be partitioned into two sets of clauses that do not
share any common relation symbols. In that case, $Q = Q' \land Q''$,
and its probability is computed using a {\em Decomposable
  Conjunction}:
\begin{align*}
  \Pr(Q) = \Pr(Q') \cdot \Pr(Q'')
\end{align*}
Finally, if none of the above cases apply to the CNF query $Q = C_1
\land C_2 \land \cdots \land C_k$, then the algorithm tries to find a
set of separator variables $x_1, \ldots, x_k$ (one for each clause).
If it finds them, then the probability is given by a {\em Decomposable
  Universal Quantifier}:
\begin{align*}
  \Pr(Q) = \prod_{a \in \text{Domain}} \Pr(C_1[a/x_1] \land \cdots \land  C_k[a/x_k])
\end{align*}

\renewcommand{\lstlistingname}{}
\begin{figure}[htbp]
\hrule
\vspace{.5em}
\footnotesize
\begin{tabbing}
\scriptsize 
Algorithm \algoname{} \\
\scriptsize
Input: \= \scriptsize Ranked and shattered query $Q$ \\
 \> \scriptsize Probabilistic DB with domain $D$
\normalsize
\end{tabbing}
\vspace{-1em}
\begin{tabbing}
\tiny
  Output: \= \scriptsize $\Pr(Q)$
\normalsize
\end{tabbing}
\hrule
\begin{lstlisting}[label=algorithm_pseudo,mathescape]
Step 0: If $Q$ is a single ground literal $\ell$, return its probability $\Pr(\ell)$ in PDB
Step 1: Write $Q$ as a union-CNF: $Q = Q_1 \lor Q_2 \lor \cdots \lor Q_m$
Step 2: If $m > 1$ and $Q$ can be partitioned into two sets $Q = Q' \lor Q''$ with disjoint relation symbols, return $1-(1-\Pr(Q_1))\cdot(1-\Pr(Q_2))$
		/* Decomposable Disjunction */
Step 3: If $Q$ cannot be partitioned, return $\sum_{s \subseteq [m]} \Pr(\bigwedge_{i \in s} Q_i)$
		/* Inclusion/Exclusion - perform cancellations before recursion */
Step 4: Write $Q$ in CNF: $Q = C_1 \land C_2 \land \cdots \land C_k$
Step 5: If $k > 1$, and $Q$ can be partitioned into two sets $Q= Q' \land Q''$ with disjoint relation symbols, return $\Pr(Q_1)\cdot \Pr(Q_2)$
		/* Decomposable Conjunction */
Step 6: If $Q$ has a separator variable, return $\prod_{a \in D} \Pr(C_1[a/x_1] \land \cdots \land  C_k[a/x_k])$
		/* Decomposable Universal Quantifier */
Otherwise $\mathbf{FAIL}$
\end{lstlisting}
\hrule
\vspace{5pt}
\caption{{\bf Algorithm for Computing $\Pr(Q)$}}
\label{fig:alg}
\end{figure}
We prove our first main result:
\begin{theorem}
  One of the following holds:  (1) either
  \algoname{} fails on $Q$, or (2) for any domain size $n$ and a PDB
  consisting of probabilities for the ground tuples, \algoname{}
  computes $\Pr(Q)$ in polynomial time in $n$.
\end{theorem}
\begin{proof}
  (Sketch) The only step of the algorithm that depends on the domain
  size $n$ is the decomposable universal quantifier step; this also reduces by 1 the
  arity of every relation symbol, since it substitutes it by the same
  constant $a$.  Therefore, the algorithm runs in time $O(n^k)$, where
  $k$ is the largest arity of any relation symbol.  We note that the
  constant behind $O( \cdots )$ may be exponential in the size of the
  query $Q$.
\end{proof}

%% file: mainResults.tex
\section{MAIN COMPLEXITY RESULT}
\label{sec:mainresults}

In this section we describe our main technical result of the paper:
that the algorithm is complete when restricted to a certain class of
CNF queries.

We first review a prior result, to put ours in perspective.
\citet{dalvi2012dichotomy} define an algorithm for Monotone DNF
(called Unions Of Conjunctive Queries), which can be adapted to
Monotone CNF; that adaptation is equivalent to \algoname{} restricted
to Monotone CNF queries.  \citet{dalvi2012dichotomy} prove:

\begin{theorem} \label{th:dichotomy:jacm} If algorithm \algoname{} FAILS on a Monotone CNF query
  $Q$, then computing $\Pr(Q)$ is
  \#P-hard.
\end{theorem}

However, the inclusion of negations in our query language increases
significantly the difficulty of analyzing query complexities.  Our
major technical result of the paper extends
\autoref{th:dichotomy:jacm} to a class of CNF queries with negation.

Define a {\em Type-1} query to be a CNF formula where each clause has
at most two variables denoted $x,y$, and each atom is of one of the
following three kinds:
\begin{itemize}
\item[--] Unary symbols $R_1(x), R_2(x), R_3(x), \ldots$
\item[--] Binary symbols $S_1(x,y), S_2(x,y), \ldots$ 
\item[--] Unary symbols $T_1(y), T_2(y), \ldots$ 
\end{itemize}
or the negation of these symbols.

Our main result is:

\begin{theorem} \label{th:dichotomy:type1}
  For every Type-1 query $Q$, if algorithm \algoname{} FAILS then
  computing $\Pr(Q)$ is \#P-hard.
\end{theorem}

The proof is a significant extension of the techniques used by
\citet{dalvi2012dichotomy} to prove \autoref{th:dichotomy:jacm}; we
give a proof sketch in \autoref{sec:exampleproof} and include the full
proof in the appendix.

%% file: examples.tex
\section{PROPERTIES OF \algoname{}}
\label{sec:examples}

We now describe several properties of \algoname{}, and the
relationship to other lifted inference formalisms.

\subsection{NEGATIONS CAN LOWER THE COMPLEXITY}

The presence of negations can lower a query's complexity, and our
algorithm exploits this.  To see this, consider the following query
{
\begin{align*}
Q
\,\, =& \,\,(\texttt{Tweets}(x) \vee  \neg \texttt{Follows}(x,y)) \\
& \quad \wedge
     (\texttt{Follows}(x,y) \vee \neg \texttt{Leader}(y))
\end{align*}
}
The query says that if $x$ follows anyone then $x$ tweets, and that
everybody follows the leader\footnote{To see this, rewrite the query
  as $(\texttt{Follows}(x,y) \Rightarrow \texttt{Tweets}(x)) \wedge 
    (\texttt{Leader}(y) \Rightarrow \texttt{Follows}(x,y))$.
}.

Our goal is to compute the probability $\Pr(Q)$, knowing the
probabilities of all atoms in the domain.  We note that the two
clauses are dependent (since both refer to the relation
\texttt{Follow}), hence we cannot simply multiply their probabilities;
in fact, we will see that if we remove all negations, then the
resulting query is \#P-hard; the algorithm described
by~\citet{dalvi2012dichotomy} would immediately get stuck on this
query.  Instead, \algoname{} takes advantage of the negation, by first
computing the prime implicate:
\begin{align*}
  \texttt{Tweets}(x) \vee \neg \texttt{Leader}(y)
\end{align*}
which is a disconnected clause (the two literals use disjoint logical
variables, $x$ and $y$ respectively).  After applying distributivity
we obtain:
\begin{align*}
  Q \equiv & (Q \wedge (\texttt{Tweets}(x))) \vee (Q \wedge (\neg \texttt{Leader}(y))) \\
    \equiv & Q_1 \vee Q_2
\end{align*}
and \algoname{} applies the inclusion-exclusion formula:
\begin{align*}
  \Pr(Q) = & \Pr(Q_1) + \Pr(Q_2) - \Pr(Q_1 \wedge Q_2)
\end{align*}
After simplifying the  three queries, they become:
\begin{align*}
  Q_1 = & (\texttt{Follows}(x,y) \vee \neg \texttt{Leader}(y)) \wedge  (\texttt{Tweets}(x)) \\
  Q_2 = & (\texttt{Tweets}(x) \vee  \neg \texttt{Follows}(x,y)) \wedge (\neg \texttt{Leader}(y))\\
  Q_1\wedge Q_2 = &   (\texttt{Tweets}(x)) \wedge (\neg \texttt{Leader}(y))
\end{align*}
The probability of $Q_1$ can now be obtained by multiplying the
probabilities of its two clauses; same for the other two queries.  As
a consequence, our algorithm computes the probability $\Pr(Q)$ in
polynomial time in the size of the domain and the PDB.

If we remove all negations from $Q$ and rename the predicates we get the
following query:
\begin{align*}
  h_1 = & (R(x) \vee S(x,y)) \wedge (S(x,y) \vee T(y))
\end{align*}
\citet{dalvi2012dichotomy} proved that computing the probability of
$h_1$ is \#P-hard in the size of the PDB.  Thus, the query $Q$ with
negation is {\em easy}, while $h_1$ is hard, and our algorithm takes
advantage of this by applying resolution.

\subsection{ASYMMETRIC WEIGHTS CAN INCREASE THE COMPLEXITY}

\citet{van2011completeness} has proven that any query with at most two logical variables per clause is domain-liftable.  Recall that this means that one can compute its probability in PTIME in the size of the domain, in the symmetric case, when all tuples in a relation have the same probability.  
However, queries with at most two logical variables per
clause can become \#P-hard when computed over asymmetric
probabilities, as witnessed by the query $h_1$ above.

\subsection{COMPARISON WITH PRIOR LIFTED FO-CIRCUITS}

\citet{DBLP:conf/ijcai/BroeckTMDR11,van2013lifted} introduce FO d-DNNF circuits, to
compute symmetric WFOMC problems.  An FO d-DNNF is a circuit whose nodes are one of the
following: decomposable conjunction ($Q_1 \wedge Q_2$ where $Q_1, Q_2$
do not share any common predicate symbols), deterministic-disjunction
($Q_1 \vee Q_2$ where $Q_1 \wedge Q_2 \equiv \texttt{false}$),
inclusion-exclusion, decomposable universal quantifier (a type of $\forall x, Q(x)$),
and deterministic automorphic existential quantifier.  The latter is an operation that is
specific only to structures with symmetric weights, and therefore does
not apply to our setting.  We prove that our algorithm can compute all
formulas that admit an FO d-DNNF circuit.

\begin{fact}
  If $Q$ admits an FO d-DNNF without a deterministic automorphic existential quantifier, then \algoname{} computes $\Pr(Q)$ in PTIME
  in the size of the PDB.
\end{fact}

The proof is immediate by noting that all other node types in the FO d-DNNF
have a corresponding step in \algoname{}, except for
deterministic disjunction, which our algorithm computes using
inclusion-exclusion: $\Pr(Q_1 \vee Q_2) = \Pr(Q_1) + \Pr(Q_2) - \Pr(Q_1 \wedge
Q_2) = \Pr(Q_1) + \Pr(Q_2)$ because $Q_1 \wedge Q_2 \equiv
\texttt{false}$.  However, our algorithm is strictly more powerful than FO d-DNNFs for the asymmetric WFOMC task, as we explain next.

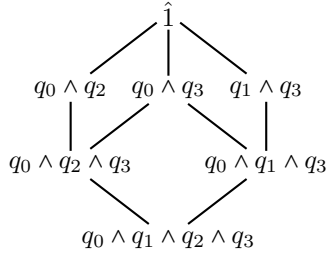
\begin{figure*}
\center
\begin{tikzpicture}
    \node (1) at (0,0) {$\hat{1}$};
    \node (Q1) at (-1.3,-1)  {$q_0 \land q_2$};
    \node (Q2) at (0,-1) {$q_0 \land q_3$};
    \node (Q3) at (1.3,-1) {$q_1 \land q_3$};
    \node (Q1Q2) at (-1.3,-2)  {$q_0 \land q_2 \land q_3$};
    \node (Q2Q3) at (1.3,-2) {$q_0 \land q_1 \land q_3$};
    \node (Q1Q2Q3) at (0,-3) {$q_0 \land q_1 \land q_2 \land q_3$};

    \draw [black,  thick, shorten <=-2pt, shorten >=-2pt] (1) -- (Q1);
    \draw [black, thick, shorten <=-2pt, shorten >=-2pt] (1) -- (Q2);
    \draw [black, thick, shorten <=-2pt, shorten >=-2pt] (1) -- (Q3);
    \draw [black, thick, shorten <=-2pt, shorten >=-2pt] (Q1) -- (Q1Q2);
    \draw [black, thick, shorten <=-2pt, shorten >=-2pt] (Q2) -- (Q1Q2);
    \draw [black, thick, shorten <=-2pt, shorten >=-2pt] (Q2) -- (Q2Q3);
    \draw [black, thick, shorten <=-2pt, shorten >=-2pt] (Q3) -- (Q2Q3);
    \draw [black, thick, shorten <=-2pt, shorten >=-2pt] (Q1Q2) -- (Q1Q2Q3);
    \draw [black, thick, shorten <=-2pt, shorten >=-2pt] (Q2Q3) -- (Q1Q2Q3);
\end{tikzpicture}
\caption{Lattice for $Q_w$.  The bottom query is
  \#P-hard, yet all terms in the inclusion/exclusion formula that
  contain this term cancel out, and $\Pr(Q_W)$ is
  computable in PTIME. }
\label{fig:qwlattice}
\end{figure*}

\subsection{CANCELLATIONS IN INCLUSION/EXCLUSION}
\label{subsec:cancellations}
\vspace{-1em}
We now look at a more complex query.  
First, let us denote four
simple queries:
\begingroup
\allowdisplaybreaks
\begin{align*}
  q_0 & =(R(x_0) \lor S_1(x_0, y_0)) \\
  q_1 & =  (S_1(x_1, y_1) \lor S_2(x_1, y_1)) \\
  q_2 & = (S_2(x_2, y_2) \lor S_3(x_2, y_2)) \\
  q_3 & =  (S_3(x_3,y_3) \lor T(y_3))
\end{align*}%
\endgroup
\citet{dalvi2012dichotomy} proved that their conjunction, i.e.\ the query $h_3 = q_0 \wedge q_1 \wedge q_2 \wedge q_3$, is \#P-hard in data complexity.  Instead of $h_3$, consider:
\begin{align*}
  Q_W = (q_0 \lor q_1) \land (q_0 \lor q_3) \land (q_2 \lor q_3)
\end{align*}
There are three clauses sharing relation symbols, hence we cannot
apply a decomposable conjunction.  However, each clause is
disconnected, for example $q_0$ and $q_1$ do not share logical
variables, and we can thus write $Q_W$ as a disjunction.
After removing redundant terms:
\begin{align*}
  Q_W = (q_0 \land q_2) \lor (q_0 \land q_3) \lor (q_1 \land q_3)
\end{align*}
Our algorithm applies the inclusion/exclusion formula:
\begin{align*}
  \Pr&(Q_W) = \Pr(q_0 \land q_2) + \Pr(q_0 \land q_3) + \Pr(q_1 \land q_3) \\
      & - \Pr(q_0 \land q_2 \land q_3) - \Pr(q_0 \land q_1 \land q_3) - \Pr(q_0  \land \cdots \land q_3) \\
      & + \Pr(q_0 \land \cdots \land q_3)
\end{align*}
At this point our algorithm performs an important step: it cancels out
the last two terms of the inclusion/exclusion formula.  Without this key step, no algorithm could compute the
query in PTIME, because the last two terms are precisely $h_3$, which
is \#P-hard.  To perform the cancellation the algorithm needs to first
check which FOL formulas are equivalent, which, as we have seen, is
decidable for our language (\autoref{th:equivalence}).  Once the
equivalent formulas are detected, the resulting expressions can be
organized in a lattice, as shown in \autoref{fig:qwlattice}, and the
coefficient of each term in the inclusion-exclusion formula is
precisely the lattice's M\"obius function~\citep{stanley-combinatorics-1997}.

%% file: extensions.tex
\section{EXTENSIONS AND LIMITATIONS}
\label{sec:extensions}

We describe here an extension of \algoname{} to symmetric WFOMC,
and also prove that a complete characterization of the complexity of
all FOL queries is impossible.

\subsection{SYMMETRIC WFOMC}
\label{sec:extensions:symmetric}

Many applications of SRL require weighted model counting for FOL
formulas over PDBs where the probabilities are associated to relations
rather than individual tuples.  That is, $\texttt{Friend}(a,b)$ has
the same probability, independently of the constants $a,b$ in the
domain.  In that symmetric WFOMC case, the model has a large number of symmetries
(since the probabilities are invariant under permutations of constants), and lifted inference algorithms may further exploit these
symmetries.  \citep{van2013lifted} employ one
operator that is specific to symmetric probabilities, called {\em
  atom counting}, which is applied to a unary predicate
$R(x)$ and iterates over all possible worlds of that predicate.
Although there are $2^n$ possible worlds for $R$, by conditioning on
any world, the probability will depend only on the cardinality $k$ of
$R$, because of the symmetries.  Therefore, the
system iterates over $k= 0, n$, and adds the conditional probabilities
multiplied by ${n \choose k}$.  For example, consider the following
query:
\begin{align}
  H = (\neg R(x) \vee S(x,y) \vee \neg T(y)) \label{eq:h}
\end{align}
Computing the probabilities of this query is \#P-hard
(\autoref{th:dichotomy:type1}).  However, if all tuples $R(a)$ have
the same probability $r\in [0,1]$, and similarly tuples in $S, T$ have
probabilities $s,t$, then one can check that\footnote{Conditioned on
  $|R| =k$ and $|T|=l$, the query is true if $S$ contains at least
  one pair $(a,b) \in R \times T$.}  
\begin{align*}
  \Pr(H) = \sum_{k, l = 0,n}
r^k\cdot (1-r)^{n-k}\cdot t^l\cdot
(1-t)^{n-l}\cdot(1-s^{kl})
\end{align*}

Denote Sym-\algoname{} the extension of \algoname{} with a
deterministic automorphic existential quantifier operator.  The question is whether this
algorithm is complete for computing the probabilities of queries over
PDBs with symmetric probabilities.  Folklore belief was that this existential quantifier operator was the only operator required
to exploit the extra symmetries available in PDBs with symmetric
probabilities.  For example, all queries in
\citet{DBLP:conf/ijcai/BroeckTMDR11} that can be computed in PTIME over
symmetric PDBs have the property that, if one removes all unary
predicates from the query, then the residual query can be computed in
PTIME over asymmetric PDBs.

We answer this question in the negative.  Consider the following query:
\begin{align*}
  Q = & (S(x_1,y_1) \lor \neg S(x_1,y_2) \lor \neg S(x_2,y_1) \lor S(x_2,y_2))
\end{align*}
Here, we interpret $S(x,y)$ as a {\em typed relation}, where the
values $x$ and $y$ are from two disjoint domains, of sizes $n_1, n_2$
respectively, in other words, $S \subseteq [n_1] \times [n_2]$.
\begin{theorem} We have that
\begin{itemize}
  \item[--] $\Pr(Q)$ can be computed in time polynomial in the size of a symmetric PDB with probability $p$ as
  $\Pr(Q) = f(n_1,n_2) + g(n_1,n_2)$ where:
  \begin{align*}
    f(n_1, 0) &= 1 \\
    f(n_1, n_2) &= \sum_{k=1}^{n_1} {n_1 \choose k} p^{kn_2}  g(n_1-k,n_2) \\
    g(0, n_2) &= 1 \\
    g(n_1, n_2) &= \sum_{\ell=1}^{n_2} {n_2 \choose \ell} (1-p)^{n_1\ell}  f(n_1,n_2-\ell)
  \end{align*}
  \item[--] Sym-\algoname{} fails to compute $Q$.
\end{itemize}
\end{theorem}

The theorem shows that new operators will be required for symmetric WFOMC.  We note that it is currently open whether computing $\Pr(Q)$ is \#P-hard in the case of asymmetric WFOMC.
\begin{proof}
  Denote $D_x, D_y$ the domains of the variables $x$ and $y$.  Fix a
  relation $S \subseteq D_1 \times D_2$.  We will denote $a_1, a_2,
  \ldots \in D_1$ elements from the domain of the variable $x$, and
  $b_1, b_2, \ldots \in D_2$ elements from the domain of the variable
  $y$.  For any $a, b$, define $a \prec b$ if $(a,b) \in S$, and $a
  \succ b$ if $(a,b) \not\in S$; in the latter case we also write $b
  \prec a$.  Then, (1) for any $a,b$, either $a \prec b$ or $b \prec
  a$, (2) $\prec$ is a partial order on the disjoint union of the
  domains $D_1$ and $D_2$ iff $S$ satisfies the query $Q$. 
  The first
  property is immediate.  To prove the second property, notice that
  $Q$ states that there is no cycle of length 4: $x_1 \prec y_2 \prec
  x_2 \prec y_1 \prec x_1$.  By repeatedly applying resolution between
  $Q$ with itself, we derive that there are no cycles of length 6, 8,
  10, etc.  Therefore, $\prec$ is transitive, hence a partial
  order. Any finite, partially ordered set has a minimal element,
  i.e.\ there exists $z$ s.t.\ $\forall x$, $x\not\prec z$.  Let $Z$ be
  the set of all minimal elements, and denote $X = D_1 \cap Z$ and $Y
  = D_2 \cap Z$.  Then exactly one of $X$ or $Y$ is non-empty, because
  if both were non-empty then, for $a \in X$ and $b \in Y$ we have
  either $a \prec b$ or $a \succ b$ contradicting their minimality.
  Assuming $X\neq\emptyset$, we have (a) for all $a \in X$ and $b \in
  D_2$, $(a,b) \in S$, and (b) $Q$ is true on the relation $S' = (D_1
  - X) \times D_2$.   
  This justifies the recurrence
  formula for $\Pr(Q)$.
\end{proof}

\subsection{THE COMPLEXITY OF ARBITRARY FOL QUERIES}

We conjecture that, over asymmetric probabilities (asymmetric WFOMC), our algorithm is
complete, in the sense that whenever it fails on a query, then the
query is provably \#P-hard.  Notice that \algoname{} applies only to a
fragment of FOL, namely to CNF formulas without function symbols, and
where all variables are universally quantified.  We present here an
impossibility result showing that a complete algorithm cannot exist
for general FOL queries.  We use for that a classic result by
Trakhtenbrot~\citep{DBLP:books/sp/Libkin04}:

\begin{theorem}[Finite satisfiability] 
The problem: ``given a FOL
  sentence $\Phi$, check whether there exists a finite model for
  $\Phi$'' is undecidable.
\end{theorem}
From here we obtain:
\begin{theorem}
  There exists no algorithm that, given any FOL sentence $Q$ checks
  whether $\Pr(Q)$ can be computed in PTIME in the asymmetric PDB size.
\end{theorem}
\begin{proof}
By reduction from the finite satisfiability problem. Fix
the hard query $H$ in Eq.(\ref{eq:h}), for which the counting problem is
\#P-hard. Recall that $H$ uses the symbols $R, S, T$. Let $\Phi$ be any
formula over a disjoint relational vocabulary (i.e. it doesn't
use $R, S, T$). We will construct a formula $Q$, such that computing
$\Pr(Q)$ is in PTIME iff $\Phi$ is unsatisfiable in the finite: this
proves the theorem. To construct $Q$, first we modify $\Phi$ as
follows. Let $P(x)$ be another fresh, unary relational symbol.
Rewrite $\Phi$ into $\Phi'$ as follows: replacing every $(\exists x
. \Gamma)$ with $(\exists x . P(x) \land \Gamma)$ and every $(\forall x
. \Gamma)$ with $(\forall x . P(x) \Rightarrow \Gamma)$ (this is not equivalent
to the guarded fragment of FOL); leave the rest of the formula
unchanged. Intuitively, $\Phi'$ checks if $\Phi$ is true on the
substructure defined by the domain elements that satisfy $P$. More
precisely: for any database instance $I$, $\Phi'$ is true on $I$ iff $\Phi$
is true on the substructure of $I$ defined by the domain elements
that satisfy $P(x)$. Define the query $Q = (H \land \Phi')$. We now
prove the claim.

If $\Phi$ is unsatisfiable then so is $\Phi'$, and therefore $Pr(Q) = 0$
is trivially computable in PTIME. 

If $\Phi$ is satisfiable, then fix any deterministic database
instance $I$ that satisfies $\Phi$; notice that $I$ is deterministic,
and $I \models \Phi$. Let $J$ be any probabilistic instance over the
vocabulary for $H$ over a domain disjoint from $I$. Define $P(x)$ as
follows: $P(a)$ is true for all domain elements $a \in I$, and $P(b)$ is
false for all domain elements $b \in J$. Consider now the
probabilistic database $I \cup J$. (Thus, $P(x)$ is also
deterministic, and selects the substructure $I$ from $I \cup J$;
therefore, $\Phi'$ is true in $I \cup J$.) We have $Pr(Q) = Pr(H \land \Phi') = Pr(H)$,
because $\Phi'$ is true on $I \cup J$. Therefore,
computing $Pr(Q)$ is \#P-hard. Notice the role of $P$: while $I$
satisfies $\Phi$, it is not necessarily the case that $I \cup J$
satisfies $\Phi$. However, by our construction we have ensured that
$I \cup J$ satisfies $\Phi'$.
\end{proof}

%% file: exampleProof.tex

\section{PROOF OF THEOREM \ref{th:dichotomy:type1}}
\label{sec:exampleproof}

The proof of \autoref{th:dichotomy:type1} is based on a reduction from
the \#PP2-CNF problem, which is defined as follows.  Given two
disjoint sets of Boolean variables $X_1, \dots, X_{n}$ and $Y_1,
\dots, Y_{n}$ and a bipartite graph $E \subseteq [n] \times [n]$,
count the number of satisfying truth assignments $\#\Phi$ to the
formula: $\Phi = \bigwedge_{(i, j) \in E} (X_i \lor Y_j)$.
\citet{provan1983complexity} have shown that this problem is \#P-hard.

More precisely, we prove the following: given any Type-1 query $Q$ on
which the algorithm \algoname{} fails, we can reduce the \#PP2-CNF
problem to computing $\Pr(Q)$ on a PDB with domain size $n$.  The
reduction consists of a combinatorial part (the construction of
certain gadgets), and an algebraic part, which makes novel use of the
concepts of algebraic independence \citep{yu1995relations} and
annihilating polynomials \citep{kayal2009complexity}.  We include the
latter in the appendix, and only illustrate here the former on a
particular query of Type-1.

We illustrate the combinatorial part of the proof on the following
query $Q$:
\[
(R(x) \lor \lnot S(x,y) \lor T(y)) \land (\lnot R(x) \lor S(x,y) \lor \lnot T(y))
\]
To reduce $\Phi$ to the problem of computing $\Pr(Q)$, we construct a structure
with unary predicates $R$ and $T$ and binary predicate $S$,
with active domain $[n]$.

We define the tuple probabilities as follows.  Letting $x, y, a, b \in
(0,1)$ be four numbers that will be specified later, we define:
\begin{align*}
\Pr(R(i)) &= x \\
\Pr(T(j)) &= y \\
\Pr(S(i, j)) &= \left\{
  \begin{array}{lr}
    a & \text{ if }(i, j) \in E \\
    b & \text{ if }(i, j) \not \in E
  \end{array}
\right.
\end{align*}
Note this PDB does not have symmetric probabilities: in fact, over structures with
symmetric probabilities one can compute $\Pr(Q)$ in PTIME.

Let $\theta$ denote a valuation of the variables in $\Phi$. Let $E_\theta$ denote the event
$\forall i.(R(i) = \true \texttt{ iff } \theta(X_i) = \true) \\
\land \forall j.(T(j) = \true \texttt{ iff } \theta(Y_j) = \true)$.

$E_\theta$ completely fixes the unary predicates $R$ and $T$ and leaves $S$ unspecified.
Given $E_\theta$, each Boolean variable corresponding to some
$S(x,y)$ is now independent of every other $S(x', y')$. 
In general, given an assignment of $R(i)$ and $T(j)$, we examine the four formulas
that define the probability that the query is true on $(i, j)$: $F_1 = Q[R(i) = 0, T(j) = 0]$,
$F_2 = Q[R(i) = 0, T(j) = 1]$,
$F_3 = Q[R(i) = 1, T(j) = 0]$,
$F_4 = Q[R(i) = 1, T(j) = 1]$.

For $Q$, $F_1,F_2,F_3,F_4$ are as follows:
\begin{align*}
F_1 = \lnot S(i, j) && F_2 = F_3 = \true && F_4 = S(i,j)
\end{align*}
Denote $f_1, f_2, f_3, f_4$ the arithmetization of these Boolean formulas:
\begin{align*}
& f_1 = \left\{
  \begin{array}{lr}
    1-a & \text{ if }(i, j) \in E \\
    1-b & \text{ if }(i, j) \not \in E
  \end{array}
\right. \\
& f_4 =  \left\{
  \begin{array}{lr}
    a & \text{ if }(i, j) \in E \\
    b & \text{ if }(i, j) \not \in E
  \end{array}
\right. \
\end{align*}
Note that $f_2 = f_3 = 1$ and do not change $\Pr(Q)$. 

Define the parameters $k,l,p,q$ of $E_\theta$ as
 $k =$ number of $i$'s s.t. $R(i)=\true$,
 $l =$ number of $j$'s s.t. $T(j)=\true$,
 $p =$ number of $(i,j) \in E$ s.t. $R(i)=T(j)=\true,$
 $q =$ number of $(i,j) \in E$ s.t. $R(i)=T(j)=\false$.

Let $N(k,l,p,q) = $ the number of $\theta$'s that have parameters $k,l,p,q$.
If we knew all $(n+1)^2(m+1)^2$ values of $N(k,l,p,q)$, we could
recover $\#\Phi$ by summing over $N(k,l,p,q)$ where
$q = 0$.
That is, $\#\Phi = \sum_{k,l,p} N(k,l,p,0)$.

We now describe how to solve for $N(k,l,p,q)$, completing the hardness proof
for $\Pr(Q)$.

We have $\Pr(E_\theta) = x^k (1-x)^{n-k} y^l (1-y)^{n-l}$ and
$\Pr(Q | E_\theta) = a^p (1-a)^q  b^{kl - p} (1-b)^{(n-k)(n-l) - q}$. 
Combined, these give the following expression for $\Pr(Q)$:
\begin{align*}
\Pr(Q) & = \sum_{\theta} \Pr(Q|E_{\theta}) \Pr(E_\theta) \\
  & = (1-b)^{n^2} (1-x)^n (1-y)^n \sum_{k,l,p,q} T \tag{1} \label{eqn:probQ}
\end{align*}
where:
\begin{align*}
T =  & N(k,l,p,q) \ast (a/b)^p [(1-a)/(1-b)]^q \\
&                    [x/(1-b)^n (1-x)]^k [y/(1-b)^n \\
&                    (1-y)]^l [b(1-b)]^{kl} \\
= &     N(k,l,p,q) \ast A^p B^q X^k Y^l C^{kl} \tag{2} \label{eqn:termT}
\end{align*}
Equations \eqref{eqn:probQ} and \eqref{eqn:termT} express $\Pr(Q)$ as a
polynomial in $X, Y, A, B, C$ with unknown coefficients $N(k,l,p,q)$.
Our reduction is the following: we choose $(n+1)^2(m+1)^2$ values for
the four parameters $x,y,a,b \in (0,1)$, consult an oracle for $\Pr(Q)$
for these settings of the parameters, then solve a linear system of
$(n+1)^2(m+1)^2$ equations in the unknowns $N(k,l,p,q)$.  The crux of
the proof consists of showing that the matrix of the system is
non-singular: this is far from trivial, in fact had we started from a
PTIME query $Q$ then the system {\em would} be singular.  Our proof
consists of two steps (1) prove that we can choose $X,Y,A,B$
independently, in other words that the mapping $(x,y,a,b) \mapsto
(X,Y,A,B)$ is locally invertible (has a non-zero Jacobian), and (2)
prove that there exists a choice of $(n+1)^2(m+1)^2$ values for
$(X,Y,A,B)$ such that the matrix of the system is non-singular: then,
by (1) it follows that we can find $(n+1)^2(m+1)^2$ values for
$(x,y,a,b)$ that make the matrix non-singular, completing the proof.
For our particular example, Part (1)  can be verified by direct
computations (see \autoref{subsec:independent}); for general queries this
requires \autoref{subsec:algVarieties}.  Part (2) for this query is 
almost as general as for any query and we show it in 
\autoref{subsec:invertible}.

%% file: relatedWork.tex
\section{RELATED WORK}
\label{sec:relatedWork}

The algorithm and complexity results of \citet{dalvi2012dichotomy},
which apply to positive queries, served as the starting point for our 
investigation of asymmetric WFOMC with negation.
See \citet{suciu2011probabilistic} for more background on their work.
The tuple-independence assumption of PDBs presents 
a natural method for modeling asymmetric WFOMC.
Existing approaches for PDBs can express complicated correlations \citep{jha2010lifted,jha2012probabilistic}
but only consider queries without negation.

Close in spirit to the goals of our work are \citet{van2011completeness} and
\citet{jaeger2012liftability}. 
They introduce a formal definition of lifted inference and 
describe a powerful knowledge compilation technique for WFOMC.
Their completeness results for first-order
knowledge compilation on a variety of query classes motivate our
exploration of the complexity of lifted inference.
\citet{cozman2009complexity} analyze the complexity of probabilistic description logics.

Other investigations of evidence in lifted inference include
\citet{van2012conditioning}, who allow arbitrary hard evidence on unary relations, \citet{bui2012exact}, who allow asymmetric soft evidence on a single unary
relation, and \citet{van2013complexity}, who allow evidence of bounded Boolean rank. Our model allows entirely asymmetric probabilities and evidence.

%% file: conclusion.tex

\section{CONCLUSION}
\label{sec:conclusion}

Our first contribution is the algorithm \algoname{} for counting
models of arbitrary CNF sentences over asymmetric probabilistic structures. 
Second, we prove
a novel dichotomy result that completely classifies a subclass of CNFs
as either PTIME or \#P-hard. Third, we describe capabilities of
\algoname{} not present in prior lifted inference techniques. Our final
contribution is an extension of our algorithm to symmetric WFOMC and
a discussion of the impossibility of establishing a dichotomy for all
first-order logic sentences.

%% file: appendix.tex

\appendix
\section{APPENDIX}

\subsection{RANKING QUERIES}

We show here that every query can be {\em ranked} (see
\autoref{subsec:preproc}), by modifying both the query Q and the
database.  Each relational symbol $R$ of arity $k$ is replaced by
several symbols, one for each possible order of its attributes.  We
illustrate this for the case of a binary relation symbol $R(x,y)$.
Given a domain of size $n$ and probabilities $\Pr(R(a,b))$ for all
tuples in $R$, we create three new relation symbols, $R_1(x,y),
R_2(x), R_3(y,x)$, and define their probabilities as follows:

{
\begin{align*}
\Pr(R_1(a,b)) =  & \left\{
  \begin{array}{lr}
    \Pr(R(a,b)) & \text{ if } a < b \\
    0 & \text{ otherwise }
  \end{array}
  \right. \\
\Pr(R_2(a)) = & \Pr(R(a,a)) \\
\Pr(R_3(b,a)) =  & \left\{
  \begin{array}{lr}
    \Pr(R(a,b)) & \text{ if } a > b \\
    0 & \text{ otherwise }
  \end{array}
  \right. \\
\end{align*}
}

Then, we also modify the query as follows.  First, we replace every
atom $R(x,y)$ with $R_1(x,y) \lor R_2'(x,y) \lor R_3(y,x)$, and every
negated atom $\neg R(x,y)$ with $\neg R_1(x,y) \land \neg R_2'(x,y)
\land \neg R_3(y,x)$, re-write the query in CNF, then replace each
clause containing some atom $R_2'(x,y)$ with two clauses: in the first
we substitute $y := x$, and in the second we replace $R_2'(x,y)$ with
\texttt{false} (which means that, if $R_2'(x,y)$ was positive then we
remove it, and if it was negated then we remove the entire clause). 
\autoref{subsec:preproc} provides an example of this procedure.

\input{invertible}

\input{independent}
 
\input{appendixDef}

\input{proofOutline}

\input{algIndependence}

\input{caseAnalysis}

\input{case1}

\input{case2}

\input{multipleUnary}

\input{zigzag}

\input{algVarieties}

%% file: invertible.tex

\subsection{PROVING THE MATRIX OF SECTION \ref{sec:exampleproof} IS INVERTIBLE}
\label{subsec:invertible}

Let $M(m_1, m_2, n_1, n_2)$ be the matrix whose entries are:
\[
     A_u^p    B_v^q    X_w^k    Y_z^l    C_{uv}^{kl}
\]

    where the row is $(p,q,k,l)$ and column is $(u,v,w,z)$
    and the ranges are:
\begin{align*}
&         p,u = 0,..,m_1-1 \\
&         q,v = 0,..,m_2-1 \\
&         k,w = 0,..,n_1-1 \\
&         l,z = 0,..,n_2-1 \\
\end{align*}

Given a vector $X_0, X_1, \dots, X_{n-1}$ denote V(X) the
determinant of their Vandermonde matrix: $V(X) = \prod_{0 \leq
k < k' < n}(X_k - X_{k'})$

\begin{lemma}
\label{lem:matrixlemma1}
If $m_1=m_2=1$ then
\[
   det(M) = C_{00}^{n_1 n_2 (n_1-1) (n_2-1)/4}  V^{n_2}(X) V^{n_1}(Y)
\]
\end{lemma}
\begin{proof} 
The matrix $M(1,1,n_1,n_2)$ has the following entries:
\[
        X_w^k   Y_z^l  C_{00}^{kl}
\]
   
All elements in row $(k,l)$ have the common factor $C_{00}^{kl}$.  After
we factorize it from each row, the remaining matrix is a Kronecker
product of two Vandermonde matrices.
\end{proof}

\begin{lemma}
\label{lem:matrixlemma2}
\begin{gather*}
   det(M(m_1, m_2, n_1, n_2)) = \\
       \prod_{u > 0}(A_u - A_0)^{m_2  n_1  n_2} \\
          det(M(1, m_2, n_1, n_2)) \\
          det(M(m_1-1, m_2, n_1, n_2))
\end{gather*}

Where in $M(m_1-1, m_2, n_1, n_2)$ instead of $A_0, \dots, A_{m_1-2}$ we have
$A_1, \dots, A_{m_1-1}$, i.e. the index $u$ is shifted by one, and similarly
in $C_{uv}$ the index $u$ is shifted by one.
\end{lemma}

\begin{proof}
We eliminate $A_0$, similarly to how we would eliminate it from
a Vandermonde matrix: subtract from row $(p+1,q,k,l)$ the row $(p,q,k,l)$
multiplied by $A_0$; do this bottom up, and cancel $A_0$ in all rows,
except the rows of the form $(0,q,k,l)$.  We only need to be careful
that, when we cancel $A_0$ in row $(p+1, q, k, l)$ we use the same $q,k,l$
to determine the row $(p,q,k,l)$.

For an illustration we show below these two rows $(p,q,k,l)$ and
$(p+1,q,k,l)$, and the two columns, $(0, v_0, w_0, z_0)$ and $(u,v,w,z)$:

In the original matrix:
\[
\begin{pmatrix}
 \vdots & \vdots & \vdots & \vdots & \vdots \\
\dots & T_1 & \dots & T_2 & \dots \\
 \vdots & \vdots & \vdots & \vdots & \vdots \\
\dots & T_3 & \dots & T_4 & \dots \\
 \vdots & \vdots & \vdots & \vdots & \vdots \\
\end{pmatrix}
\]

Where:
\begin{align*}
& T_1 = A_0^p    B_{v_0}^q    X_{w_0}^k    Y_{z_0}^l    C_{0,v_0}^{kl} \\
& T_2 = A_u^p    B_v^q    X_w^k    Y_z^l    C_{uv}^{kl} \\
& T_3 = A_0^{p+1}    B_{v_0}^q    X_{w_0}^k    Y_{z_0}^l    C_{0,v_0}^{kl} \\
& T_4 = A_u^{p+1}    B_v^q    X_w^k    Y_z^l    C_{uv}^{kl}
\end{align*}

Subtract the first row times $A_0$ from the second row, and obtain:

\[
\begin{pmatrix}
 \vdots & \vdots & \vdots & \vdots & \vdots \\
\dots & T_1 & \dots & T_2 & \dots \\
 \vdots & \vdots & \vdots & \vdots & \vdots \\
\dots & 0 & \dots & T_4 - A_0  T_2 & \dots \\
 \vdots & \vdots & \vdots & \vdots & \vdots \\
\end{pmatrix}
\]

Where:
\begin{align*}
& T_4 - A_0  T_2 = (A_u-A_0)  A_u^p    B_v^q    X_w^k    Y_z^l    C_{uv}^{kl}
\end{align*}

Repeat for all rows in this order: $(m_1-1,q,k,l), (m_1-2,q,k,l),
\dots,(1,q,k,l)$, and for all combinations of $q,k,l$.  Let's examine the
resulting matrix.

Assume that the first $m_2  n_1  n_2$ rows are of the form $(0,q,k,l)$.  Also,
assume that the first $m_2  n_1  n_2$ columns are the form $(0,v,w,z)$
(permute if necessary)

Therefore the matrix looks like this:

\[
\left(
\begin{array}{c|c}
    m_1  &   \dots \\ \hline
    0   &   M'
\end{array}
\right)
\]

Where:
\begin{itemize}

\item The top-left $m_2  n_1  n_2$ rows and columns are precisely $m_1 =
   M(1,m_2,n_1,n_2)$.  Notice that this matrix does not depend on $A$.
   All entries below it are 0.

\item Therefore, $det(M) = det(m_1)    det(M')$

   where $M'$ is the bottom right matrix (what remains after removing
   the first $m_2  n_1  n_2$ rows and columns).   This follows from
   a theorem on expanding determinants

\item $M'$ has a factor $(A_u - A_0)$ in every column $(u,v,w,z)$.  
   Factorize this common factor,
   noting that it occurs $m_2  n_1  n_2$ times (for all combinations of
   $v,w,z$).  Thus:
\[
       det(M') = \prod_u (A_u - A_0)^{m_2  n_1  n_2} det(M'')
\]

   where $M''$ is the matrix resulting from $M'$ after factorizing.

\item The entries of $M''$ are precisely:
\[
    A_u^p    B_v^q    X_w^k    Y_z^l    C_{uv}^{kl}
\]
   where $p=0,\dots,m_1-2$, and $u=1,\dots,m_1-1$ and the other indices have the same
   range as before.

\item Therefore, $M'' = M(m_1-1, m_2, n_1, n_2)$, with the only change that the
   index $u$ is shifted by one.
  
\end{itemize}
  
\end{proof}

Lemmas \ref{lem:matrixlemma1} and \ref{lem:matrixlemma2} prove that 
$det(M) \ne 0$ whenever all the $A$'s, the $B$'s,
the $X$'s, and the $Z$'s are distinct, and all $C_{uv} \ne 0$.  Thus, our
determinant in \autoref{sec:exampleproof} is nonzero, as
$C_{uv} = (A_u - 1)(B_v - 1)/(B_v - A_u)$.

%% file: independent.tex

\subsection{PROVING THE FUNCTIONS OF SECTION \ref{sec:exampleproof} ARE LOCALLY INVERTIBLE}
\label{subsec:independent}

In this section, we prove that the functions from the example in \autoref{sec:exampleproof}
are locally invertible:
\begin{align*}
X(x,b) =  & \frac{x}{(1-x) (1-b)^n } \\
Y(y,b) =  & \frac{y}{(1-y) (1-b)^n } \\
A(a,b) =  & \frac{a}{b} \\
B(a,b) =  & \frac{1-a}{1-b} \\
\end{align*}

We show this by computing the determinant of the Jacobian matrix of these functions. 
In the general proof, the concept of algebraic independence
replaces the notion of locally invertible.

Let $J$ be the Jacobian matrix of the vector-valued function 
$F(x,y,a,b) = (X(x,b), Y(y,b), A(a,b), B(a,b))$.
\[
J = \begin{pmatrix}
 \frac{1}{b} & \frac{-a}{b^2} & 0 & 0 \\
 \frac{-1}{1-b} & \frac{1-a}{(1-b)^2} & 0 & 0 \\
 0 & \frac{nx}{(1-x)(1-b)^{n+1}} & \frac{1}{(1-x)^2 (1-b)^n} & 0 \\
 0 & \frac{ny}{(1-y)(1-b)^{n+1}} & 0 & \frac{1}{(1-y)^2 (1-b)^n}
\end{pmatrix}
\]

The determinant of this matrix is:
\[
det(J) = \displaystyle\frac{b-a}{(1-y)^2 (1-x)^2 b^2 (1-b)^{2(n+1)}}
\]

For any values of $x,y,a,b$ s.t. $a \ne b$, $det(J) \ne 0$. By the inverse
function theorem, $F$ is invertible in some neighborhood contained in 
$(0, 1)^4$. We pick our values of $x,y,a,b$ to lie within this neighborhood.

%% file: appendixDef.tex

\subsection{DEFINITIONS}

Let $Q$ be a query with a single left unary and a single right unary
symbol $U(x), V(y)$.  Let $F$ be its Boolean formula, and denote:
\begin{align*}
& F_{00} = F[0/U, 0/V] \\
& F_{01} = F[0/U, 1/V] \\
& F_{10} = F[1/U, 0/V] \\
& F_{11} = F[1/U, 1/V]
\end{align*}

With some abuse of notation we refer to these functions as $F_1, F_2, F_3, F_4$,
and their arithmetizations to multilinear polynomials as $f_1, f_2, f_3, f_4$.

Call $Q$ \emph{splittable} if $F$ has a prime implicate
consisting only of unary symbols with at least one left unary symbol 
$U$ and at least one right unary symbol $V$.  Note that if $Q$ is 
splittable, then the algorithm applies the inclusion/exclusion formula.

Call $Q$ \emph{decomposable} if $F = (F_1 \land F_2)$, where all left unary symbols
$U_i$ are in $F_1$, all right unary symbols $V_j$ are in $F_2$, and $F_1, F_2$ do not share any common
symbols (they are independent).  Note that if $Q$ is decomposable, then
the algorithm applies decomposable conjunction.

Call $Q$ \emph{immediately unsafe} if it is neither splittable nor
decomposable.  When running the algorithm on an immediately unsafe query $Q$,
the algorithm is immediately stuck.

Given queries $Q, Q'$ we say that $Q$ \emph{rewrites} to $Q'$, with notation
$Q \rightarrow Q'$,
if $Q'$ can be obtained from $Q$ by setting some symbol to $\true$  or to
$\false$, i.e. $F' = F[0/Z]$ or $F' = F[1/Z]$.

Call a query $Q$ \emph{unsafe} if it can be rewritten to some immediately
unsafe query: $Q \rightarrow \ldots \rightarrow Q'$ and $Q'$ is immediately unsafe.

Call a query $Q$ \emph{forbidden} if it is immediately unsafe, and any
further rewriting $Q \rightarrow Q'$ is to a safe query (i.e. $\Pr(Q')$ can be
computed by the algorithm, and therefore is in PTIME).

\begin{fact}Q is splittable iff one of the four functions $F_1, \dots, F_4$ is
unsatisfiable.
\end{fact}
\begin{proof}
If $Q$ is splittable then it has a prime implicate of the form
$((\lnot)U \lor (\lnot)V)$.  Then that corresponding function is 0.  For example,
suppose $Q \Rightarrow (\lnot U \lor V)$.  Then $F_{10} = F[1/U,0/V] = 0$.
\end{proof}

\begin{fact}$Q$ is decomposable iff there exists polynomials $g_0, g_1$ and $h_0,
h_1$ such that the polynomials $f_{00}, f_{01}, f_{10}, f_{11}$ factorize as follows:
\begin{align*}
&   f_{00} = g_0 h_0 \\
&   f_{01} = g_0 h_1 \\
&   f_{10} = g_1 h_0 \\
&   f_{11} = g_1 h_1 
\end{align*}
\end{fact}
\begin{proof} 
Assume $f_{00}, f_{01}, f_{10}, f_{11}$ factorize as above.  Then we have $f =
(1-u)(1-v) f_{00} + \dots + u v f_{11} = ((1-u) g_0+u g_1) ((1-v) h_0+v h_1)$
proving that $Q$ is decomposable.  The converse is immediate.
\end{proof}
 
Our hardness proof requires the following background on
multivariate polynomials:

\begin{definition}[Annihilating Polynomial]
Let $f_1, \dots, f_n$ be multivariate polynomials. An
annihilating polynomial is a polynomial $A(z_1,\dots,z_n)$ such that the
following identity holds:  $A(f_1, \dots, f_n) = 0$.
\end{definition}

\begin{definition}[Algebraic Independence]
A set of polynomials $f_1, \dots, f_n$ is algebraically independent if there does not
exist an annihilating polynomial that annihilates $f_1, \dots, f_n$.
If $f_1, \dots, f_n$ have an annihilating polynomial, then the Jacobian determinant
$Det(J(f_1, \dots, f_n))=0$ everywhere. In this case, the polynomials are said to be
algebraically dependent.
\end{definition}

\begin{proposition}
If $f_1, \dots, f_n$ are over $n-1$ variables, then they have an annihilating polynomial.
Equivalently, $f_1, \dots, f_n$ are algebraically dependent.
\end{proposition}

\begin{proposition}
If $f_1, \dots, f_n$ have an annihilating polynomial, and
any $n-1$ are algebraically independent, then there exists a unique
irreducible annihilating polynomial $A$ for $f_1, \dots, f_n$.
\end{proposition}

\begin{proposition}
If the Jacobian $J(f_1, \dots, f_n)$ has rank less than $n$, then 
$f_1, \dots, f_n$ have an annihilating polynomial.
\end{proposition}

Our proofs consider annihilating polynomials for
the four Boolean functions resulting from a query Q conditioned
on its unary left and right predicates.

Consider the following two examples of annihilating polynomials:
\begin{itemize}
	\item  If $Q$ decomposes: $f_1=g_0h_0$, $f_2=g_0h_1$, $f_3=g_1h_0$, $f_4=g_1h_1$,
   			then the annihilating polynomial is $A = f_1f_4 - f_2f_3 = 0$
	\item  Suppose $f_1 = x_1 + x_2 - x_1x_2$, $f_2 = x_1x_2$, $f_3=x_1$.  Then
   			$A = (f_1 + f_2 - f_3)f_3 - f_2 = 0$
\end{itemize}

We also need the following: (1) The ideal generated by $f_1, \dots, f_n$,
denoted $\langle f_1, \dots, f_n \rangle $, is the set of polynomials of the form $f_1 h_1 +
\dots + f_n h_n$, for arbitrary $h_1, \dots, h_n$.  (2) The variety of an ideal
$I$ is $V(I) = \{a | \forall f \in I, f[a/x] = 0\}$. In particular,
$V(f_1,\dots,f_n)$ is the variety of $\langle f_1,\dots,f_n \rangle$ and consists of all
common roots of $f_1,\dots,f_n$. (3) Hilbert's Nullstellensatz: if $V(I)
\subseteq V(f)$ then there exists $m$ s.t. $f^m \in I$.  We only need a
very simple consequence: if $p$ is irreducible and $V(p) \subseteq
V(f)$, then $f \in \langle p \rangle $. In other words, $f$ is divisible by $p$.

%% file: proofOutline.tex

\subsection{OUTLINE OF HARDNESS PROOF}
\label{subsec:proofOutline}

Given a forbidden query $Q$, we prove
hardness by reduction from \#PP2CNF (see \autoref{sec:exampleproof}).
Given a PP2CNF formula $\Phi$:
\[
       \Phi = \bigwedge_{(i,j) \in E} (X_i \lor Y_j)
\]

Where $E \subseteq [n] \times [n]$, we set the probabilities as
follows:
\begin{align*}
& \Pr(U(i)) = u \\
& \Pr(V(j)) = v \\
& \Pr(X_1(i,j)) = x_1, \Pr(X_2(i,j) = x_2, \dots if (i,j) \in E \\
& \Pr(X_1(i,j)) = y_1, \Pr(X_2(i,j) = y_2, \dots if (i,j) \not\in E
\end{align*}

Fix an assignment $\theta: \{X_1, \dots, X_n, Y_1, \dots, Y_n\} \rightarrow \{0,1\}$.

Define the following parameters of $\theta$:
\begin{align*}
&   k = \text{number of $i$'s s.t. } X_i=1 \\
&   l = \text{number of $j$'s s.t. } Y_j=1 \\
&   q = \text{number of $(i,j) \in E$ s.t. } X_i = 0, Y_j = 0 \\
&   r = \text{number of $(i,j) \in E$ s.t. } X_i = 0, Y_j = 1 \\
&   s = \text{number of $(i,j) \in E$ s.t. } X_i = 1, Y_j = 0 \\
&   p = \text{number of $(i,j) \in E$ s.t. } X_i = 1, Y_j = 1
\end{align*}

Let $N(k,l,q,r,s,p)$ = number of assignments $\theta$ with these parameters.

By repeating the calculations we did for the example query, and omitting a constant
factor, we obtain:
\[
\Pr(Q) = \sum_{k,l,q,r,s,p}  N(k,l,q,r,s,p) A^q  B^r  C^s D^p X^k  Y^l  H^{kl}
\]

Where:
\begin{align*}
&  A = f_{00}(x_1,x_2, \dots) / f_{00}(y_1,y_2, \dots) \\
&  B = f_{01}(x_1,x_2, \dots) / f_{01}(y_1,y_2, \dots) \\
&  C = f_{10}(x_1,x_2, \dots) / f_{10}(y_1,y_2, \dots) \\
&  D = f_{11}(x_1,x_2, \dots) / f_{11}(y_1,y_2, \dots) \\
&  H = \text{depends on } A,B,C,D \\
&  X = \text{depends on } A,B,C,D \text{ and } u \\
&  Y = \text{depends on } A,B,C,D \text{ and } v
\end{align*}

As in the example of \autoref{sec:exampleproof}, we use an oracle for $\Pr(Q)$ repeatedly to construct
a system of linear equations and solve for $N(k,l,q,r,s,p)$ in polynomial
time. From here we derive $\#\Phi$.

To do this, we must prove that the matrix $M$ of the resulting system has
$det(M) \ne 0$.

The same technique used in \autoref{subsec:invertible} generalizes to prove
that $M$ is non-singular, as long as we can produce distinct 
values for $A, B, C,$ and $D$. This establishes the following:

\begin{fact} Let $m = |E|$. Consider four sequences of m+1 distinct
numbers:
\begingroup
\allowdisplaybreaks
\begin{align*}
&   A_u & u=0,\dots,m \\
&   B_v & v=0,\dots,m \\ 
&   C_w & w=0,\dots,m \\
&   D_z & z=0,\dots,m
\end{align*}%
\endgroup

Suppose that for every combination of $u,v,w,z$ we can find
probabilities $x_1, x_2, \dots, y_1, y_2, \dots$ s.t. $A_u =
f_{00}(x_1,x_2,\dots)/f_{00}(y_1,y_2,\dots)$, 
$B_v = f_{01}(x_1,x_2,\dots)/f_{01}(y_1,y_2,\dots)$, etc.
Then $det(M) \ne 0$.
\end{fact}

Thus, to prove that $\Pr(Q)$ is \#P-hard it suffices to prove that the four
functions $A, B, C, D$ are invertible: that is, given their output
values $A_u, \dots, D_z$, we must find inputs $x_1,x_2, \dots,y_1,y_2,\dots$ s.t. 
when the functions are applied to those inputs they result in the desired
values.

Clearly, $A,B,C,D$ are not invertible in two trivial cases: when some of
the functions $f_{00},f_{01},f_{10},f_{11}$ are constants, or when two or more are
equivalent. There are several other special cases, detailed later. As we will
see, some of these cases may still be solved by identifying a subset of $\{A,B,C,D\}$
which is invertible, and the rest of the cases are solved by a second hardness
proof technique referred to as the zigzag construction.

Overloading terminology, we say that a query $Q$ is invertible iff $A,B,C,D$ (or a subset
thereof, if some functions are equivalent or constant) are invertible.
The case analysis of \autoref{subsec:caseanalysis} proves the following
theorem:

\begin{theorem}
Let $Q$ be a forbidden Type 1 query. Then one of the following holds:
\begin{itemize}
	\item $Q$ is invertible and we apply the hardness proof as described above
	\item $Q$ admits the zigzag construction and hardness proof of Section
		\ref{subsec:zigzag}
\end{itemize}
\end{theorem}

%% file: algIndependence.tex

\subsection{IMPLICATIONS OF ALGEBRAIC INDEPENDENCE}

Establishing the algebraic independence of the functions $A,B,C,D$ is one
of two primary challenges in the proof technique of \autoref{subsec:proofOutline}.
We discuss here how algebraic independence of the functions $f_1 g_1, f_2 g_2, f_3 g_3, f_4 g_4$
implies the invertibility of $A,B,C,D$.

\begin{theorem}
Let $Q$ be a forbidden query with two unary atoms $U, V$.
Suppose the four functions $F_{00}, F_{01}, F_{10}, F_{11}$ are distinct and non-constant
(Note that this implies that there are at least two variables $x_1, x_2$).  Then the
Jacobian of the four functions $A, B, C, D$ has rank 4.
\end{theorem}

\begin{proof}
We denote the four functions $f_1(x), f_2(x), f_3(x), f_4(x)$,
where $x= (x_1, x_2, \dots)$ is the set of variables.  Further denote
$g_1(y)=f_1[y/x], \dots, g_4(y) = f_4[y/x]$, where $y=(y_1,y_2,\dots)$ are distinct
new variables.  Recall that:
\begin{align*}
   A = & f_1(x) / g_1(y) \\
   B = & f_2(x) / g_2(y) \\
   C = & f_3(x) / g_3(y) \\ 
   D = & f_4(x) / g_4(y) 
\end{align*}

Their Jacobian has the same rank as the Jacobian of their log, which
is:
\begin{align*}
  log(A) = log(f_1)  -  log(g_1) \\ 
  log(B) = log(f_2)  -  log(g_2) \\
  log(C) = log(f_3)  -  log(g_3) \\
  log(D) = log(f_4)  -  log(g_4)
\end{align*}

The Jacobian matrix looks like this:
\[
J = \begin{pmatrix}
 \frac{1}{f_1} \frac{\partial f_1}{\partial x_1} & \frac{1}{f_1} \frac{\partial f_1}{\partial x_2} &
 	 \dots &  -\frac{1}{g_1} \frac{\partial g_1}{\partial y_1} & -\frac{1}{g_1} \frac{\partial g_1}{\partial y_2} & \dots \\
 \vdots & \vdots & \dots & \vdots & \vdots & \dots \\
 \frac{1}{f_4} \frac{\partial f_4}{\partial x_1} & \frac{1}{f_4} \frac{\partial f_4}{\partial x_2} &
 	 \dots &  -\frac{1}{g_4} \frac{\partial g_4}{\partial y_1} & -\frac{1}{g_4} \frac{\partial g_4}{\partial y_2} & \dots \\
\end{pmatrix}
\]

Each column corresponding to a y-variable has a minus sign.  Reversing
these signs, which does not change the rank of the matrix, we obtain
the Jacobian of these four functions:
\begin{align*}
	log(f_1)  +  log(g_1) \\ 
	log(f_2)  +  log(g_2) \\
    log(f_3)  +  log(g_3) \\
    log(f_4)  +  log(g_4)
\end{align*}

This Jacobian is of rank 4 iff the four functions $f_1 g_1, f_2 g_2, f_3 g_3, f_4 g_4$
are algebraically independent.
\end{proof}

%% file: caseAnalysis.tex

\subsection{CASE ANALYSIS}
\label{subsec:caseanalysis}

Queries which satisfy the assumptions of Lemma \ref{ref:lemmaVarieties} are invertible,
and we apply the hardness proof described in Section \ref{subsec:proofOutline}. We consider
the remaining queries that do not satisfy the conditions of Lemma \ref{ref:lemmaVarieties}.

These queries possess functions $f_1, f_2, f_3, f_4$ such that:
\begin{gather*}
 \forall q \in Factors(f_4) - Factors(f_3), \\
    \forall p \in Factors(f_3), \\
       V(p,q) \subseteq V(f_1*f_2)
\end{gather*}

And the same holds for all permutations of $f_1,f_2,f_3,f_4$ in the above equations.

Let:
\begin{align*} \label{eqn:f3prime} \tag{1}
  f_3' = & Factors(f_3) - Factors(f_4) \\
  f_4' = & Factors(f_4) - Factors(f_3)  \\
  f_{34} = & Factors(f_3) \cap Factors(f_4)
\end{align*}

The condition above is equivalent to:
\[
  \forall p \in f_3, q \in f_4', V(p,q) \subseteq V(f_1*f_2)
\]
and permuting $f_3, f_4$:
\[
  \forall p \in f_3', q \in f_4, V(p,q) \subseteq V(f_1*f_2)
\]
The two conditions above are equivalent to the following:
\begin{align*}
   \forall p \in f_3', q \in f_4', V(p,q) \subseteq V(f_1*f_2) \\
   \forall p \in f_3', q \in f_{34}, V(p,q) \subseteq V(f_1*f_2) \\
   \forall p \in f_{34}, q \in f_4', V(p,q) \subseteq V(f_1*f_2)
\end{align*}

In the last two cases $p,q$ have disjoint sets of variables.
We prove the following:

\begin{proposition}
\label{prop:disjointvars}
If $p, q$, are irreducible polynomials over disjoint sets of
variables, then $V(p,q) \subseteq V(f*g)$ iff $V(p,q) \subseteq V(f)$ or
$V(p,q) \subseteq V(g)$.
\end{proposition}

The proposition follows from the following lemma.

\begin{lemma}
Let $p(x), q(y)$ be irreducible polynomials, over disjoint sets
of variables $x$ and $y$ respectively.  Suppose $V(p,q) \subseteq V(f_1 f_2)$
where $f_1(x,y), f_2(x,y)$ are arbitrary polynomials.  Then at least one
of the following holds:
\begin{itemize}
  \item $V(p,q) \subseteq V(f_1)$
  \item $V(p,q) \subseteq V(f_2)$
\end{itemize}
\end{lemma}

\begin{proof}
Notice that $V(p,q) = \set{(a,b)| p(a) = 0 \land q(b) = 0}$.  In other
words, $V(p,q)$ is the cartesian product $V(p) \times V(q)$, and the assumption
of the lemma is:
\[
  \forall a \in V(p), \forall b \in V(q) \Rightarrow (a,b) \in V(f_1 f_2)
\]

We claim:
\begin{gather*} \label{eqn:onestar} 
\tag{*} \forall a \in V(p): 
       \text{either } q \text{ divides } f_1[a/x] \text{ or } \\
       q \text{ divides } f_2[a/x]
\end{gather*}

Indeed, if $a \in V(p)$, then:
\[
   \set{a} \times V(q) \subseteq V((f_1 f_2)[a/x])
\]

Thus $q$ divides $f_1[a/x] f_2[a/x]$, hence it either divides $f_1[a/x]$ or
divides $f_2[a/x]$ (because it is irreducible).

We claim that the following stronger property holds:
\begin{gather*} \label{eqn:twostar} 
\tag{**} \text{ either: } \forall a \in V(p), q \text{ divides } f_1[a/x] \\
      \text{ or: }    \forall a \in V(p), q \text{ divides } f_2[a/x]
\end{gather*}

This claim proves the lemma, because in the first case $V(p,q)
\subseteq V(f_1)$, and in the second case $V(p,q) \subseteq V(f_2)$.

We prove (\ref{eqn:twostar}) by using the remainder of dividing $f_1(x,y)$ by $q(y)$,
which we denote $g_1$.  In other words:
\[ \label{eqn:cond1}   \tag{1}    
  g_1(x,y) = sum_e c_e(x) y^e
\]

Where every exponent sequence $e$ for $y$ is ``smaller'' than the
multidegree of $g$. Formally, following standard notations for
multivariate polynomials and Gr\"{o}bner bases, fix an admissible
monomial order $<$, then $g_1$ is the normal form of $g_1$ w.r.t. $p$, that
is $f_1 \Rightarrow^{*}_{q} g_1$ and there is no $h$ s.t. $g_1 \Rightarrow_{q} h$.

Similarly, let $g_2(x,y)$ be the remainder of dividing $f_2$ by $q$:
\[ \label{eqn:cond2}   \tag{2}    
	g_2(x,y) = sum_{e'} d_{e'}(x) y^{e'}
\]

From (\ref{eqn:onestar}) we have:
\begin{gather*}
\label{eqn:plus}   \tag{+}    
   \forall a \in V(p): \\
       \text{ either: } \forall e,  c_e[a/x] = 0 \\
       \text{ or  }    \forall e', d_{e'}[a/x] = 0
\end{gather*}

This implies:
\[
   \forall a \in V(p), \forall e, e' c_e[a/x] d_{e'}[a/x] = 0
\]

Or, equivalently:
\[
   \forall e, e', \forall a \in V(p),  c_e[a/x] d_{e'}[a/x] = 0
\]

Or, still equivalently:
\[
   \forall e, e': p(x) \text{ divides } c_e(x) d_{e'}(x)
\]

Since $p(x)$ is irreducible, it implies that $p(x)$ either divides $c_e(x)$
or divides $d_{e'}(x)$.  We claim that the following holds:
\begin{gather*}
\label{eqn:plusplus}   \tag{++}  
     \text{ either: }  \forall e, p(x) \text{ divides } c_e(x) \\
     \text{ or: }    \forall e', p(x) \text{ divides } d_{e'}(x)
\end{gather*}

Suppose not. Then there exists $e$ such that $p(x)$ does not divide $c_e(x)$
and there exists $e'$ such that $p(x)$ does not divide $d_{e'}(x)$.
This is a contradiction, because we know that $p(x)$ must divide one of $c_e(x)$ or
$d_{e'}(x)$.  Property (\ref{eqn:plusplus}) immediately implies (\ref{eqn:twostar}).

\end{proof}

Intuitively, proposition \ref{prop:disjointvars} generalizes the fact that:
$V(p) \subseteq V(fg)$ implies $V(p) \subseteq V(f)$ or $V(p) \subseteq V(g)$
(because $V(p) \subseteq V(f*g)$ implies that $p$ divides $fg$, hence it divides
either $f$ or $g$, because $p$ is irreducible).

By applying this argument repeatedly we obtain $V(p,q) \subseteq V(r)$,
where $r$ is some factor of $f_1$ or $f_2$.  
Recall from equation \eqref{eqn:f3prime} that $f_3' = Factors(f_3) - Factors(f_4)$
and $f_4' = Factors(f_4) - Factors(f_3)$.
Abusing notation by
using $f_1$ to denote $Factors(f_1)$, the conditions become:
\begin{align*}
\forall p \in f_3', q \in f_{34}, & \text{ either } \\
&     (p \in f_1 \cup f_2) \text{ or } \\
 &  (q \in f_1 \cup f_2)\\
\forall p \in f_{34}, q \in f_4',  & \text{ either } \\
 &  (p \in f_1 \cup f_2) \text{ or } \\
 &  (q \in f_1 \cup f_2)
\end{align*}

These conditions are equivalent to:
\begin{align*}
 (f_3' \subseteq f_1 \cup f_2) \text{ or } (f_{34} \subseteq f_1 \cup f_2) \\
 (f_{34} \subseteq f_1 \cup f_2) \text{ or } (f_4' \subseteq f_1 \cup f_2)
\end{align*}

Indeed, suppose otherwise, i.e. there exists $p \in f_3'$ and $q \in f_{34}$
s.t. neither $p$ nor $q$ are in $f_1 \cup f_2$: then the first condition above
fails too.

Applying distributivity, these conditions are equivalent to:
\[
(f_3' \subseteq f_1 \cup f_2) \text{ and } (f_4' \subseteq f_1 \cup f_2)
\]
or
\[
f_{34} \subseteq f_1 \cup f_2
\]

In other words, the proposition fails only on queries that satisfy the
following three conditions, and all conditions obtained by permuting
$f_1, \dots, f_4$:
\[ \label{eqn:condC1} \tag{C1}
	f_3  \subseteq f_1 \cup f_2
\]
           or
\[ \label{eqn:condC2} \tag{C2}
    f_4  \subseteq f_1 \cup f_2
\]
           or
\[ \label{eqn:condC3} \tag{C3}
      \Delta(f_3,f_4) \subseteq f_1 \cup f_2
\]
Where $\Delta$ denotes the symmetric difference operator.

We can now classify the queries that do not satisfy the assumptions of Lemma
\ref{ref:lemmaVarieties} according to the following corollary:

\begin{corollary}
\label{cor:case1and2}
If a query $Q$ does not satisfy the assumptions of \autoref{ref:lemmaVarieties},
then one of the following two cases holds:
\begin{enumerate}
	\item There exists an irreducible factor $w$ that occurs in only one
of the four functions $F_1, \dots, F_4$.  Assume without loss
of generality that $w \in F_4$.
Then (\ref{eqn:condC1}) must hold, under all permutations of $f_1, f_2, f_3$.
This implies that every factor that occurs in $f_1, f_2, f_3$
occurs in at least two of them.  Therefore, these functions
look like this:
\begin{align*}
       f_1 & = p * q  * s \\
       f_2 & = p * r  * s \\
       f_3 & = q * r  * s \\
       f_4 & = w * \dots
\end{align*}

That is, $p$ contains all factors that occur in both $f_1$ and $f_2$,
likewise for $q, r, s$, and $w$ occurs only in $f_4$.  For example:
\begin{align*}
       f_1 & = x_1  x_2  (1-x_3) \\
       f_2 & = x_1  (1-x_3) \\
       f_3 & = x_2  (1-x_3) \\
       f_4 & = x_3
\end{align*}

	\item Every factor occurs in two or more functions. Then the
functions look like this:
\begin{align*}
       f_1 & = p * q * r * [rest] \\
       f_2 & = p * s * t * [rest] \\
       f_3 & = q * s * r * [rest] \\
       f_4 & = r * t * r * [rest] \\
\end{align*}
where $p$ consists of all factors that occur in both $f_1$ and $f_2$,
likewise for $q, r, s, t$, and $[rest]$ represents factors that occur in three or more
functions.
\end{enumerate}
\end{corollary}

In \autoref{subsec:case1} and \autoref{subsec:case2} we describe how queries 
of these type are handled. For all other queries, the conditions of \autoref{ref:lemmaVarieties}
are satisfied and we apply the hardness proof described in \autoref{subsec:proofOutline}.

%% file: case1.tex

\subsection{CASE 1}
\label{subsec:case1}

In this section, we prove that queries falling into case 1 of the analysis of 
\autoref{cor:case1and2} still contain an algebraically independent set
of functions such that the hardness proof of \autoref{subsec:proofOutline}
applies.

Our four functions look like:
\begin{align*}
       f_1 & = p * q  * s \\
       f_2 & = p * r  * s \\
       f_3 & = q * r  * s \\
       f_4 & = w * \dots
\end{align*}

Where $p$ is a product of factors, and similarly $q, r, s$. $w$ is any factor.
Note that $p, q, r, s$ do not share any variables, due to multilinearity.

We assume that $f_4 \ne 1$ and is distinct from each of $f_1,f_2,f_3$.

We consider the following possibilities:
\[
f_1 = f_2 = f_3 
\]
	or
\[
	f_1 = f_2, f_1 \ne f_3
\] 
	or
\[
	f_1 \ne f_2, f_1 \ne f_2, f_2 \ne f_3
\]

Note that the cases $f_1 = f_3, f_1 \ne f_2$ and $f_2 = f_3, f_2 \ne f_1$ are symmetric to 
the second case, $f_1 = f_2, f_1 \ne f_3$.

Suppose $f_1 = f_2 = f_3$. This implies that $p = q = r = 1$, and $s$ is any factor. Our functions are:
\begin{align*}
       f_1 & = s \\
       f_2 & = s \\
       f_3 & = s \\
       f_4 & = w * \dots
\end{align*}

If $s = 1$, then we can invert the unary predicates (by replacing each probability $p$ with $1 - p$)
as necessary to ensure that $f_4 = f_{00}$, and we can solve the \#PP2-CNF by summing over 
assignments where the number of clauses with end points both false is held to zero.

If $s \ne 1$, then we group $f_1,f_2,f_3$ into a single function, $f'$. 
We consider an annihilating polynomial $A$ s.t. $A(f' g', f_4 g_4) = 0$. We set $g_4 = 0$ and 
$g' \ne 0$ (by setting the factor $w$ of $f_4$ to $0$) and obtain 
$A(f', 0) = 0 \Rightarrow A = a_2 R$, a contradiction of the irreducibility of $A$. 
This shows algebraic independence of the polynomials $f' g', f_1 g_4$, allowing the hardness reduction
of \autoref{subsec:proofOutline} to proceed.

Next, suppose $f_1 = f_2, f_1 \ne f_3$. If $f_3 = 1$, then $q = r = s = 1$ and our functions are:
\begin{align*}
       f_1 & = p \\
       f_2 & = p \\
       f_3 & = 1 \\
       f_4 & = w * \dots
\end{align*}

As before, we group $f_1$ and $f_2$ and ensure (by manipulating tuple probabilities for the unary
predicates) that $f_4$ corresponds to $f_{00}$. Algebraic independence of $f_1g_1$ and $f_4g_4$ 
follows by the same argument above.

The case $f_1 = f_2 = 1$, and $f_3 \ne 1$, is impossible due to the assumed structure on our 
functions (every factor in $f_3$ also appears in either $f_1$ or $f_2$)

Consider now the case when $f_1,f_2,f_3$ are distinct. Since $s$ appears in all three functions, we ignore it for now and look at $p, q, r$. For the functions to be distinct, we must have at least two of these factors not equal to $1$ (and themselves distinct). Assume wlog that $p \ne 1, q \ne 1, p \ne q$.

Then, if $r = 1$, our functions are:
\begin{align*}
       f_1 & = p *q* s \\
       f_2 & = p *s\\
       f_3 & = q *s  \\
       f_4 & = w * \dots
\end{align*}

Since $p, q, s$ are over distinct variables, there are at least 3 distinct variables in $f_1, f_2, f_3$. 
Consider the (rectangular) Jacobian of $f_1,f_2,f_3$ with respect to $x_1,x_2,x_3$, where 
$x_1$ is chosen s.t. $x_1$ is in $p$, $x_2$ is in $q$, and $x_3$ is in $s$.

The Jacobian contains the following 3x3 sub matrix:
\[
J = \begin{pmatrix}
 q  s  \partial_p/\partial_{x_1} & p  s \partial_q/\partial_{x_2} & p  q  \partial_s/\partial_{x_3}  \\
 s  \partial_p/\partial_{x_1} & 0 & p  \partial_s/\partial_{x_3}  \\
 0 & s  \partial_q/\partial_{x_2} & q  \partial_s/\partial_{x_3}  \\
\end{pmatrix}
\]

The determinant of $J$ is:
\[
det(J) = - p  q  s * \partial_p/\partial_{x_1} * \partial_q/\partial_{x_2} * \partial_s/\partial_{x_3} \ne 0
\]

This establishes the algebraic independence of $f_1,f_2,f_3$. 

Suppose there exists an annihilating polynomial $A(a_1,a_2,a_3,a_4)$ 
s.t. $A(f_1g_1,f_2g_2,f_3g_3,f_4g_4) = 0$. We set $g_4 = 0$ (using the distinct $w$ factor) 
and set $g_1 = c_1 \ne 0, g_2 = c_2 \ne 0, g_3 = c_3 \ne 0$.
We obtain $A(c_1f_1,c_2f_2,c_3f_3,0) = 0$. It follows that $A = a_4  R$, as any terms of $A$ 
without $a_4$ imply the existence of an annihilating polynomial for $c_1f_1,c_2f_2,c_3f_3$, 
which implies an annihilating polynomial for $f_1,f_2,f_3$. Thus, by contradiction,
$f_1, f_2, f_3, f_4$ are algebraically independent.

The final case is if $r \ne 1$. Our functions are:
\begin{align*}
       f_1 & = p *q* s \\
       f_2 & = p *r *s\\
       f_3 & = q *r *s  \\
       f_4 & = w * \dots
\end{align*}

Since $p, r, q$ are over distinct variables, there are at least 3 distinct variables in $f_1, f_2, f_3$. 
As before, we consider the (rectangular) Jacobian of $f_1,f_2,f_3$ with respect to $x_1,x_2,x_3$, where 
$x_1$ is chosen s.t. $x_1$ is in $p$, $x_2$ is in $q$, and $x_3$ is in $r$.

The Jacobian contains the following 3x3 sub matrix:
\[
J = \begin{pmatrix}
 q  \partial_p/\partial_{x_1} & p  \partial_q/\partial_{x_2} & 0 \\
r  \partial_p/\partial_{x_1} & 0 & p  \partial_r/\partial_{x_3}  \\
 0 & r  \partial_q/\partial_{x_2} & q  \partial_r/\partial_{x_3}  \\
\end{pmatrix}
\]

With determinant:
\begin{gather*}
det(J) = -2 q  p r \partial_p/\partial_{x_1} * \partial_q/\partial_{x_2} * \partial_r/\partial_{x_3}
\end{gather*}

None of these terms are constantly zero, so we have that the determinant is nonzero.
Repeating the previous argument with annihilating polynomials, we prove that
$f_1, f_2, f_3, f_4$ are algebraically independent.

%% file: case2.tex

\subsection{CASE 2}
\label{subsec:case2}

We prove that queries falling into case 2 of the analysis of 
\autoref{cor:case1and2} are precisely those queries satisfying
the conditions of the zigzag construction. For these queries,
we prove hardness as described in \autoref{subsec:zigzag}.

Our four functions look like:
\begin{align*}
       f_1 & = p * q  * r \\
       f_2 & = p * s  * t \\
       f_3 & = q * s  * k \\
       f_4 & = r * t  * k
\end{align*}

Where arbitrary additional factors may be added, as long as each of these additional factors appears
in at least three of the four functions.

Only $p$ and $k$, or $q$ and $t$, or $r$ and $s$, can share variables. 
Every other pair of factors appears together in one of $f_1,f_2,f_3,f_4$, and thus must have distinct variables by multilinearity. 
Let $x$ be the variables of $p,k$, let $y$ be the variables of $q,t$, and let $z$ be the variables of $r,s$, 
with $x,y,z$ all disjoint. We have:
\begin{align*}
       f_1 & = p(x) * q(y)  * r(z) \\
       f_2 & = p(x) * t(y)  * s(z) \\
       f_3 & = k(x) * q(y)  * s(z) \\
       f_4 & = k(x) * t(y)  * s(z)
\end{align*}

Because $x, y, z$ are disjoint sets of variables, 
we can set $p(x) = k(x) = c_1 \ne 0$, $q(y) = t(y) = c_2 \ne 0$, 
and $r(z) = s(z) = c_3 \ne 0$. (If a factor is identically one, then $c_i = 1$)

This gives us:
\begin{align*}
       f_1 & = c_1 * c_2 * c_3 \\
       f_2 & = c_1 * c_2 * c_3 \\
       f_3 & = c_1 * c_2 * c_3 \\
       f_4 & = c_1 * c_2 * c_3 \\
\end{align*}

Note that any additional factors, added to at least three of the four functions, must be 
over an independent set of variables. 
Thus, we can set each such additional factor to 1 and retain the same value of $f_1,f_2,f_3,f_4$
as above.

This gives us a setting of all four functions to a constant, non-zero value. This is the precondition
for applying the zigzag construction of \autoref{subsec:zigzag}.

%% file: multipleUnary.tex

\subsection{MULTIPLE UNARY SYMBOLS}

We prove that a query with multiple left or right unary symbols can always be rewritten
to an equivalent, in terms of hardness, query with one unary symbol.

\subsubsection{Rewriting an immediately unsafe query}

We first prove that, if $Q$ is immediately unsafe, it is equivalent to a
query with only one left and one right unary symbol.

\begin{proposition}
If $Q$ is immediately unsafe and $U$ any unary symbol, then
$Q[0/U]$ is not splittable, and $Q[1/U]$ is not splittable.
\end{proposition}
\begin{proof}
Let $Q[0/U] \Rightarrow T$, where $T$ is a prime implicate consisting only
of unary symbols, with at least one $U_i$ and one $V_j$.  Then 
$Q \Rightarrow U \lor T$, and one can check that no strict subset of 
$U \lor T$ is an 
implicate of $Q$, hence $U \lor T$ is a prime implicate of $Q$, proving that 
$Q$ is splittable, a contradiction.
\end{proof}

\begin{proposition}
If $Q$ is immediately unsafe, has at least two unary
symbols $U_1, U_2$, and both $Q[0/U_1, 0/U_2]$ and $Q[1/U_1, 0/U_2]$ are
satisfiable, then at least one of the following four queries is not decomposable:
\[
	Q[0/U_1], Q[1/U_1], Q[0/U_2], Q[1/U_2]
\]
\end{proposition}
\begin{proof}
We use the following two facts:
\begin{enumerate}[(A)]
  \item If $p$ does not depend on $u$ and divides $f$, then $p$ divides both
      $f[0/u]$ and $f[1/u]$

  \item Conversely: let $u, v$ be two distinct variables, $f$ a multilinear
      polynomial, and assume $f[0/u] \ne 0$.

      Let $p'(v), p(v)$ be the unique factors of $f$ and $f[0/u]$,
      respectively, that contain $v$.

      Then, if $p'(v)$ does not depend on $u$, then $p'(v)=p(v)$.  In other
      words, the factor $p(v)$ of $f[0/u]$ is also a factor of $f$.  Notice
      that we must assume $f[0/u] \ne 0$, otherwise $p(u)$ is not uniquely
      defined.

      The same statement holds for $f[1/u]$.
\end{enumerate}

Suppose both $q[0/U_1]$ and $q[1/U_1]$ are decomposable.  Let $v$ be any
variable corresponding to a right predicate.  Since $q$ depends on $v$, 
at least one of $q[0/U_1],q[1/U_1]$ also depends on $v$, and we assume
wlog $q[0/U_1]$ depends on $v$.

Let $p(v)$ be the irreducible factor of $q[0/U_1]$ that contains $v$. By
definition, $p(v)$ does not depend on $U_2$.

From fact (A) we obtain:
\begin{align*}
    p(v) \text{ divides } q[0/U_1, 0/U_2] & \text{ and } \\
    p(v) \text{ divides } q[0/U_1, 1/U_2]
\end{align*}

Suppose now that $q[0/U_2]$ is also decomposable, and let $p'(v)$ be its
irreducible factor containing the variable $v$.  (If $q[0/U_2]$ does not
depend on v, then $p'(v) = 1$.)

From Fact (A) we also obtain:
\begin{align*}
    p'(v) \text{ divides } q[0/U_1, 0/U_2] & \text{ and } \\
    p'(v) \text{ divides } q[1/U_1, 0/U_2]
\end{align*}

We apply fact (B) to $f = q[0/U_1]$ and $f[0/U_2] = q[0/U_1, 0/U_2]$: their
factors containing $v$ are $p(v)$ and $p'(v)$ respectively, and since
$q[0/U_1, 0/U_2]\ne0$ we must have $p(v) = p'(v)$.

We apply fact (B) again to $f = q[1/U_1]$ and $f[0/U_2] = q[1/U_1, 0/U_2]$:
since the latter has the factor $p(v)$, so must the former, in other
words $p(v)$ is a factor of $q[1/U_1]$.

Therefore $p(v)$ is a factor of $q$, and does not contain any unary symbol
$U_1, U_2, \dots$.  Repeating this argument for every variable $v$, we conclude
that $q$ is decomposable, which is a contradiction.

\end{proof}

The only cases that remain to be handled are when:
\begin{align*}
    Q[0/U_1]=0 \text{ and } Q[1/U_2]=0 \text{ or } \\
    Q[0/U_1]=0 \text{ and } Q[1/U_2]=0
\end{align*}

Thus, either $Q \Rightarrow (U_1 \Leftrightarrow U_2)$, 
or $Q \Rightarrow (U_1 \Leftrightarrow \lnot U_2)$.
We treat these cases by substituting all occurrences of the 
predicate $U_2$ with $U_1$ (or $\lnot U_1$) in $Q$.  
The new query $Q'$ has the same probability as $Q$ but one fewer
unary symbol.

\subsubsection{Hardness of $Q$ after inclusion/exclusion}

We prove that if the algorithm starts with query $Q$ and reaches
an immediately unsafe query $Q'$ during an inclusion/exclusion step, there
is a sequence of deterministic rewrites from $Q$ to an immediately unsafe
query $Q''$. This shows that, if the algorithm gets stuck during an inclusion/exclusion
step while computing $\Pr(Q)$, then computing $\Pr(Q)$ is \#P-hard.

Suppose $Q$ is splittable. Then Q contains one or more splittable clauses
of the form $(L_i \lor R_i)$, where $L_i$ is the disjunction of one or more
left unary symbols and $R_i$ is the disjunction of one or more right unary symbols:
\[
  Q = (L_1 \lor R_1) \land (L_2 \lor R_2) \land \cdots \land (L_m \lor R_m) \land Q
\]

After splitting on the $(L_1 \lor R_1)$ clause and applying distributivity, $Q$ may be written:

\begin{gather*}
  Q = L_1 \land (L_2 \lor R_2) \land \cdots \land (L_m \lor R_m) \land Q \\
    \lor \\
      R_1 \land (L_2 \lor R_2) \land \cdots \land (L_m \lor R_m) \land Q \\
    =  Q_1 \lor Q_2
\end{gather*}

We may continue to split $Q_1$ and $Q_2$ into $Q_{11}, Q_{12}, Q_{21}, Q_{22}$, and so on.
Some clauses may be lost due to the introduction of redundancy,
but in general we end up with an expression for $Q$ as the disjunction of $2^m$ CNF formulas $Q_i$:
\[
  Q = Q_1 \lor Q_2 \lor \cdots \lor Q_{2^m}
\]

Each $Q_i$ is of the form:
\[
  Q_i = L_{w_1} \land \cdots \land L_{w_j} \land R_{z_1} \land \cdots \land R_{z_k} \land Q
\]
Where $w$ and $z$ define sequences mapping to $L_1, \dots, L_m$ and $R_1, \dots, R_m$.

The above expression for $Q$ in terms of the $Q_i$ is generated by the algorithm before
applying the inclusion/exclusion step. Thus, the algorithm attempts to compute $\Pr(Q)$
recursively according to the formula $\Pr(Q) = - \sum_{s \subseteq [m]} (-1)^{|s|} \Pr(\bigwedge_{i\in s}  Q_i)$.
Note that every term in this summation can be written in the general form of $Q_i$ above.
We claim that, if any term of the summation is immediately unsafe, there is a deterministic rewrite
sequence $\rho$ (setting unary symbols to true or false) that satisfies 
each $L_i$ and $R_j$ clause, such that $Q_i[\rho] = Q[\rho]$, and that $Q[\rho]$ is immediately 
unsafe. This implies that, if the algorithm gets stuck while recursively processing a query $Q$ after
an inclusion/exclusion step, $Q$ is \#P-hard.

We now prove the following proposition, from which the above claim follows immediately.

\begin{proposition}
If $Q' = L \land Q$, where $L$ is a disjunction of only left or only right
unary symbols, and $Q'$ is immediately unsafe, then there exists a unary symbol
$U_i$ in $L$ and value $\alpha \in \{0,1\}$ such that $Q'[\alpha/U_i] = \true \land Q[\alpha/U_i] = Q[\alpha/U_i]$,
and $Q[\alpha/U_i]$ is immediately unsafe.
\end{proposition}
\begin{proof}
Let $m$ be the number of positive literals in $L$ and $n$ be the number of negated literals, such
that $L$ may be written:
\[
	L = U_1 \lor \cdots \lor U_m \lor \lnot U_{m+1} \lor \cdots \lor \lnot U_{m+n}
\]

Denote by $q$ the arithmetization of the grounding of $Q'$ over a domain of size 1.

The clause $L$ in $Q'$ implies that $q$ must take the following form:
\[
	q = \sum_{s \subseteq [m+n]} \prod_{\substack{i \in s,\\ 1 \le i \le m}} u_i \prod_{\substack{j \in s,\\ m+1\le j \le m+n}} (1-u_j) f_s
\]
Which states that every term of $q$ must contain a variable corresponding to some $U_i$ in $L$.

Now, suppose that $q[1/u_i]$ is decomposable for all $1 \le i \le m$ and $q[0/u_i]$ is decomposable for
all $m+1 \le j \le m+n$. 

We can write $q[1/u_1]$ as follows:
\begin{align*}
	q[1/u_1] & = f_{\{1\}} + \sum_{s \subseteq [m+n]} \prod_{\substack{i \in s, \\ 1 \le i \le m, \\ i \ne 1}} u_i \prod_{\substack{j \in s, m+1\le j \le m+n}} (1-u_j) f_s \\
	& = s_1 t_1
\end{align*}
Where $s_1$ is a polynomial that contains every left unary variable, and $t_1$ is a polynomial that contains every right unary variable,, and the variables
of $s_1$ and $t_1$ are disjoint.

Since $t_1$ divides $q[1/u_1]$, and $t_1$ does not depend on any $u_i$, we have that $t_1$ also divides $q[1/u_1, 0/u_2, \dots, 0/u_m, 1/u_{m+1}, \dots, 1/u_{m+n}] = f_{\{1\}}$.

Repeating this process for every $u_i$, we see that $t_i$ divides $f_{i}$, for every $1 \le i \le m + n$.
Finally, each $t_i$ must divide $q[1/u_1, \dots, 1/u_m, 0/u_{m+1}, \dots, 0/u_{m+n}]$, or $t_i = t_j$ for all $i, j$. Let
$t$ denote this common value.

We may repeat this process for all subsets of $[m+n]$ of size two, obtaining that $t$ must also divide those, and continue
for all subsets of size $3, 4, \dots, m+n$, until we have that $t$ divides $f_s$ for all $s \subseteq [m+n]$. From here,
we see that $t$ divides $q$, contradicting the assumption that $Q'$ was not decomposable.

\end{proof}

%% file: zigzag.tex

\subsection{ZIGZAG CONSTRUCTION}
\label{subsec:zigzag}

The zigzag construction is a technique used in \citep{dalvi2012dichotomy} to
prove the \#P-hardness of positive queries. The essence of their technique
is that, given a query $Q$, one can construct a $DB$ such that $\Pr(Q) \equiv \Pr(Q')$,
where $Q' = Q_1 \land Q_2 \land \cdots$; essentially, $Q'$ is the conjunction of multiple copies of $Q$,
each over distinct relational atoms except for their unary atoms, which are connected in a linear
chain from $Q_1 \rightarrow Q_2 \rightarrow \cdots$. This is an essential tool in their reduction from 
$\#\Phi$ for positive queries. The full construction is quite
complex, and we refer to their work for complete details.

We note here one crucial assumption behind the zigzag construction that prevents it
from applying directly to queries with negation: with a monotone query, by setting
tuple probabilities to 0 or 1 as appropriate, it is simple to ensure that
$\Pr(Q')$ does not depend on unwanted edges between atoms of $KB$, e.g., between a unary
atom of $Q_i$ and a unary atom of $Q_{i+2}$. If the query is monotone, we simply set all
such probabilities to 1 (in the CNF case) and the undesired components of the query
vanish. However, this is not guaranteed to work for queries with negation: we must consider all
possible assignments to tuples in the domain, and thus, in general, the claim that
$\Pr(Q) \equiv \Pr(Q')$ fails. The motivation for our analysis in \autoref{subsec:case2}
is that, when the probabilities on each unwanted edge of the query $Q'$ can be set 
to some non-zero constant $c_i$, we can treat all unwanted components of the expression 
for $\Pr(Q')$ as constant factor $c_0$, dependent on the size of the domain and the constants
$c_1, \dots, c_k$. This gives us $\Pr(Q) \equiv c_0 \Pr(Q')$, allowing us in these cases
to use the zigzag construction to prove hardness for queries with negation.

%% file: algVarieties.tex

\subsection{ALGEBRAIC VARIETIES}
\label{subsec:algVarieties}

\begin{lemma}
\label{ref:lemmaVarieties}
Suppose there exists two factors $p \in Factors(f_3)$, $q \in
Factors(f_4)$ such that the following hold:
\begin{align*}
   (a) & &   V(q) \not\subseteq V(f_3) \\
   (b) & &  V(p,q) \not\subseteq V(f_1) \cup V(f_2)
\end{align*}

If $k_1, k_2$ are algebraically independent, then the polynomials
$f_1  k_1$, $f_2  k_2$,
$f_3  k_3$, $f_4 k_4$ are
algebraically independent.
\end{lemma}

\begin{proof}

Suppose the contrary, that there exists an annihilating polynomial:
\[
   A(f_1k_1, \dots, f_4k_4) = 0
\]

From (b) we derive that there exists some value $a \in V(p,q)$ such that:
\[
    f_1[a/x] \ne 0, f_2[a/x] \ne 0, f_3[a/x] = f_4[a/x] = 0
\]

From (a) and (b) we derive that $V(q)$ is not included in $V(f_1) \cup
V(f_2) \cup V(f_3)$. Otherwise $V(q) \subseteq V(f_1f_2f_3)$ and by
Hilbert's Nullstellensatz: $(f_1f_2f_3)^m \in \langle q \rangle $, hence $q$ is a factor
of $(f_1f_2f_3)^m$, hence it is a factor of either $f_1$, $f_2$, or $f_3$,
violating either (b) or (a)  Thus, there exists a value $b \in V(q)$
such that:
\[
    f_1[b/x] \ne 0, f_2[b/x] \ne 0, f_3[b/x] \ne 0, f_4[b/x] = 0
\]

We claim that it is possible to choose $a$ and $b$ such that they are
consistent, in other words we claim that $p[b/x]$ has some free
variables (not set by $b$) such that we can obtain $a$ by setting those
variables to some constants.

This is easiest to see using the quotient construction.  If $R[x]$
denotes the ring of multivariate polynomials over $x$, then $R[x]/q(x)$ is
the quotient ring.

For any polynomial $f(x) \in R[x]$, its equivalence class is denoted
$[f(x)] \in R[x]/q(x)$.  Setting $q=0$ means, technically, replacing
every polynomial $f$ with [$f$].  Note that $[q(x)] = 0$, which implies
$[f_4]=0$, and we have $[f_1],[f_2],[f_3] \ne 0$, because $V(q)$ is not a
subset of $V(f_1f_2f_3)$.

For a polynomial $F(x,y) \in R[x,y]$, its equivalence class $[F(x,y)]$ is
obtained by writing $F$ as a sum of monomials:
\[
   F = sum_e C_e(x)  y^e
\]

Then $[F(x,y)] = sum_e [C_e(x)]  y^e$.  Therefore, $[f_1k_1]=[f_1]k_1$, and
likewise for $f_2,f_3,f_4$.

Denoting $B(z_1,z_2,z_3) = A(z_1,z_2,z_3,0)$, we cannot have $B \equiv 0$ because then
$A$ would be reducible.  Thus:
\begin{align*}
0 = & [A(f_1k_1, \dots, f_4k_4)]  \\
= & A([f_1]k_1, [f_2]k_2, [f_3]k_3, 0) \\
                          = & B([f_1]k_1, [f_2]k_2, [f_3]k_3) \\
                          =  & B1([f_1]k_1, [f_2]k_2, [f_3]k_3)
\end{align*}

Since $B([f_1]k_1,\dots)$ is identically $0$, there exists an irreducible
polynomial $B_1$ s.t. $B_1([f_1]k_1, \dots)$ is identically $0$.

Next, we set $[p]=0$.  Formally, we obtain this by constructing the
new quotient ring $(R[x]/\langle q \rangle)/\langle p \rangle $, and mapping every polynomial $[f]$ to
$[[f]]$.  We have $[[p]]=0$, hence $[[f_3]]=0$.  We claim that $[[f_1]], [[f_2]]
\ne 0$.  Indeed, suppose $[[f_1]]=0$, then $[f_1] \in \langle [p] \rangle $, which is
equivalent to $f_1 \in \langle p,q \rangle $, but this contradicts (b).  Therefore:
\begin{align*}
0 = & B_1([[f_1]]k_1, [[f_2]]k_2, [[f_3]]k_3) \\
  = & B_1([[f_1]]k_1, [[f_2]]k_2, 0)
\end{align*}

Since $[[f_1]], [[f_2]]$ are non-zero, we can substitute all their
variables with constants s.t. $[[f_1]] = c_1 \ne 0$, $[[f_2]] = c_2 \ne 0$:
\[
0 = B_1(c_1k_1, c_2k_2, 0)
\]

This is a contradiction, proving the claim.
\end{proof}